\documentclass[journal,10pt]{IEEEtran}
\usepackage{amsmath}
\usepackage[pdftex]{graphicx}
\usepackage{epsfig}
\usepackage{url}
\usepackage{psfrag, dblfloatfix, fixltx2e}
\usepackage{ifpdf}
\usepackage{cite}
\usepackage{siunitx}
\usepackage{textcomp}
\usepackage{color}
\usepackage{tabularx,colortbl}
\graphicspath{{figures/Appendix/}}
\usepackage{accents} % For align command
\usepackage{amssymb} % \in
\usepackage[caption=false,font=footnotesize,labelfont=sf,textfont=sf]{subfig}
\usepackage{color}
\usepackage{epstopdf}
\usepackage{morefloats}
\usepackage{helvet}
\usepackage{multirow}
\usepackage{float}
\usepackage{gensymb}
\usepackage{setspace}
\usepackage{placeins}
\usepackage{amsthm}
\usepackage{tikz}
\usetikzlibrary{spy}

\usepackage[usestackEOL]{stackengine}
\usepackage{scalerel}

	\DeclareMathSymbol{\widehatsym}{\mathord}{largesymbols}{"62}

	\DeclareMathSymbol{\widetildesym}{\mathord}{largesymbols}{"65}

	\usepackage{pgf}
	\usepackage{bm}
	\usepackage{algorithm,algorithmic}
	\usepackage{tkz-euclide}
	
	\newtheorem{proposition}{\textbf{Proposition}}
	\usetikzlibrary{positioning}
	\usetkzobj{all}
	
	\newtheorem{theorem}{Theorem}

	\newcommand\bigzero{\makebox(0,0){\text{\Huge$0$}}}
	\newcommand\bigast{\makebox(0,0){\text{\Huge$\ast$}}}

	\def\bI{\ensuremath{{\bf I}}}

	\def\bW{\ensuremath{{\bf W}}}
	\def\bV{\ensuremath{{\bf V}}}
	\def\bH{\ensuremath{{\bf H}}}
	\def\bQ{\ensuremath{{\bf Q}}}
	\def\bR{\ensuremath{{\mathbb{R}}}}
	\def\bX{\ensuremath{{\bf X}}}
	\def\bY{\ensuremath{{\bf Y}}}
	\def\bZ{\ensuremath{{\bf Z}}}

	\def\be{\ensuremath{{\bf e}}}
	\def\by{\ensuremath{{\bf y}}}
	\def\bw{\ensuremath{{\bf w}}}
	\def\bz{\ensuremath{{\bf z}}}
	\def\0{\ensuremath{{\bf 0}}}
	\def\tr{\ensuremath{{\mathrm{Tr}}}}
	\def\diag{\ensuremath{{\mathrm{diag}}}}

	\def\bS{\ensuremath{{\bf S}}}

	\def\cA{\ensuremath{{\mathcal A}}}
	\def\cL{\ensuremath{{\mathcal L}}}

	\def\balpha{\ensuremath{\bm{\alpha}}}
	\def\bgamma{\ensuremath{\bm{\gamma}}}
	\def\balphas{\ensuremath{\bm{\alpha}^\ast}}
	
	\def\bZ{\ensuremath{{\bf Z}}}
	\def\bZs{\ensuremath{{\bf Z^\ast}}}
	\def\bH{\ensuremath{{\bf H}}}
	\def\bHs{\ensuremath{{\bf H^\ast}}}
	\def\bws{\ensuremath{{\bf w^\ast}}}
	\def\yhat{\ensuremath{\widehat{\bf y}}}
	\def\bU{\ensuremath{{\bf U}}}

	\tikzstyle myBG=[line width=3pt,opacity=1.0]
	
	\DeclareMathAlphabet{\mathbcal}{OMS}{cmsy}{b}{n}
	
\begin{document}
		
		\title{Blind Image Deblurring Using \\ Row-Column Sparse Representations}
		
		\author{\normalsize Mohammad Tofighi, \emph{Student Member, IEEE}, Yuelong Li, \emph{Student Member, IEEE}, and Vishal Monga, \emph{Senior Member, IEEE \vspace{-0.5cm}}
			%\thanks{Manuscript received January 19, 2015; }%revised December 27, 2012.}
			\thanks{M. Tofighi, Y. Li, and V. Monga are with Department of Electriccal Engineering, The Pennsylvania State University, University Park,
				PA, 16802 USA, Emails: tofighi@psu.edu, yul200@psu.edu, vmonga@engr.psu.edu
			}% <-this % stops a space
			\thanks{Research was supported by a US National Science Foundation CAREER Award to V. Monga.}
		}

		%\markboth{Journal of \LaTeX\ Class Files,~Vol.~11, No.~4, December~2012}%
		\markboth{Accepted to IEEE Signal Processing Letters, December~2017}%
		{Tofighi  \MakeLowercase{\textit{et al.}}: Blind Image Deblurring Using Row-Column Sparse Representations}
		

		\maketitle

		\begin{abstract}
			Blind image deblurring is a particularly challenging
			inverse problem where the blur kernel is unknown and must be estimated en route to recover the deblurred image. The problem is of strong practical relevance since many imaging devices such as cellphone cameras, must rely on deblurring algorithms to yield satisfactory image quality. Despite significant research effort, handling large motions remains an open problem. In this paper, we develop a new method called Blind Image Deblurring using Row-Column Sparsity (BD-RCS) to address
			this issue. Specifically, we model the outer product of kernel and image
			coefficients in certain transformation domains as a rank-one matrix, and
			recover it by solving a rank minimization problem. Our central contribution then includes solving {\em two new} optimization problems involving row and column sparsity to automatically determine blur kernel and image support sequentially.  The kernel and image can then be recovered through a singular value decomposition (SVD).
			Experimental results on linear motion deblurring demonstrate that BD-RCS can yield better results than state of the art, particularly when the blur is caused by large motion. This is confirmed both visually and through quantitative measures. 
		\end{abstract}
		%
		%\begin{keywords}
		%Blind image deconvolution, motion blur, rank minimization, point spread function, PSF estimation
		%\end{keywords}
		%
		\section{Introduction}
		\label{sec:intro}
		\IEEEPARstart{I}{Mage} blurring may come from different sources: optical defocusing, camera
		shaking, etc and frequently contributes to loss of captured image quality. DSLR cameras
		usually have image stabilizers  either in the camera body or in the lenses.
		However, for cellphone and compact cameras, because of size
		limitations such mechanical stabilizers are hard to incorporate. In such scenarios,
		image processing algorithms are required to remove blurring artifacts
		in the post-processing stage. Mathematically image blurring can be well modelled as
		discrete convolution and image deblurring algorithms generally require solving
		an image deconvolution problem. 
		
		Depending on whether the Point Spread Function (PSF), aka blur kernel is known or otherwise, image deconvolution methods
		can be divided into two categories: non-blind image deconvolution where the PSF is exactly given and blind
		image deconvolution where only the blurred image is observed. The area of
		non-blind image deconvolution is mature with demonstrated practical success~\cite{Dong,BM3Dblur,Portilla,Chantas,Bioucas,Beck}. Blind image deconvolution on the other hand presents a more formidable practical challenge, since the blur kernel must be automatically estimated.
		
		\textbf{Related Works:} Early works in blind deconvolution include Ayers and Dainty's work in \cite{Ayers88} and the
		blind image deconvolution method based on Richardson-Lucy's method in
		\cite{RichardsonLucy95}. In~\cite{Ayers88} an
		iterative method similar to Wiener filter is employed in the Fourier domain, with a positivity constraint on the image, while
		in \cite{RichardsonLucy95}, a probabilistic solution based on Bayesian estimation is proposed.
		
		Recent blind deconvolution methods can be divided into two groups:  the first one follows an
		alternating minimization scheme \cite{Shan08,Krishnan11,Xu_2013_CVPR,Ren16}, i.e., solving for either the PSF or the image
		while fixing the other iteratively until convergence. The successes of these methods crucially rely on proper choices of regularizers.
		For instance, Shan {\it et al.}~\cite{Shan08} employed logarithmic density of the image
		gradient to exploit the edge information while Krishnan {\it et
			al.}~\cite{Krishnan11} employed a normalized sparsity measure.  Xu {\it et
			al.}~\cite{Xu_2013_CVPR} employed truncated $\ell_2$-norm as a practical
		approximation of the $\ell_0$-norm.  These methods are usually successful in recovering
		the latent images under relatively small motions; however, when confronted with large motions their performance may degrade due to edge distortions. Recently, Ren {\it et al.}~\cite{Ren16} designed a weighted nuclear norm to exploit the non-local patch similarities, especially along salient edge structures. However, this method may fail to produce clear results when rich textures are present. 
		
		The second
		one follows a two-stage scheme \cite{Fergus06,Pan16,Zuo16}, i.e., first estimating the PSF and then solving
		a non-blind deconvolution problem using the estimated kernel. A representative example is Fergus \textit{et.\ al}'s method \cite{Fergus06}. A drawback of it is the occasional ringing artifacts in the resulting images. More recent
		techniques~\cite{Xu2010,Xu_2013_CVPR,Cai12} focus on dealing with large motions by pre-processing the
		image gradients to filter out misleading edge information. Cho
		and Lee~\cite{SCho_deblur_2009} employed shock filter to detect salient edges, and a coarse-to-fine iterative refinement scheme to recover large kernels. Xu \textit{et al.} \cite{Xu2010} mask the edges based on a customized metric and recover the kernel using the masked edges only.
		
		There have been notable recent breakthroughs in understanding the optimization problem involved in solving for the blur kernel and the deblurred image. Levin \textit{et al.}~\cite{Levin11} recommend estimating the kernel before the latent image instead of jointly estimating both to rule out trivial solutions.
		Perrone \textit{et al.} in~\cite{Perrone16}
		verified this finding experimentally, and further found that
		alternating minimization over the total-variation (TV) regularized non-convex
		cost function can avoid converging to a trivial solution. Other works based on TV regularization include~\cite{Tofighi16,Chan98,Babacan09}.  From a general blind deconvolution perspective (not specifically image deblurring), Ahmed {\it et
			al.}~\cite{Ahmed14} proposed an analytical approach to convert blind
		deconvolution into a rank-one matrix recovery problem. Their work offers theoretical guarantees of recovery but is based on somewhat unrealistic prior knowledge about both the locations of nonzero PSF coefficients in pixel domain (blur kernel support) and the locations of primary coefficients of the image in certain transformation domains (image support). 
		
		\noindent \textbf{Motivation and Contributions:} From a performance standpoint, improving deblurred image quality in the face of large motion is an outstanding open challenge. The work of  Ahmed {\it et
			al.}~\cite{Ahmed14} offers promise but is hard to realize in practice due to the requirements of exact knowledge about image and PSF support. Motivated by this dichotomy, we develop a novel image deblurring method called Blind Image Deblurring using Row-Column Sparse Representations (BD-RCS). Like Ahmed {\em at al} \cite{Ahmed14}, our work formulates a rank-one matrix recovery problem but we set up {\em two new} optimization problems involving row and column sparsity to automatically determine blur kernel and image support respectively.  Note that in the analytical development, we don't pose any assumptions on the types and shapes of the kernel and thus BD-RCS is versatile across several practical blur models. In this work, we represent the image in the Haar wavelet domain but this transform is a flexible parameter in our work and its exact choice could be driven by the experimental scenario.
		
		\noindent \textbf{Reproducibility:} All the results in this paper are completely reproducible. We also share our code publicly at \cite{webpage}.\vspace{-0.2cm}
		
		\section{Proposed method}
		\label{sec:Proposed}
		\subsection{Formulation of Blind Deconvolution as Rank Minimization}
		\label{subsec:review}
		For completeness, we first briefly review the technique of formulating blind
		deconvolution as a rank-one matrix recovery problem originally proposed
		in~\cite{Ahmed14}. We also discuss its fundamental limitations towards
		practical applications. Recall the mathematical model for the blurred image:\vspace{-0.2cm}
		\begin{equation}
			\bm{y} = \bm{x}\ast\bm{k}+\bm{n},\vspace{-0.2cm}
			\label{problem}
		\end{equation}
		where $\bm{x}$ is the latent image, $\bm{k}$ is the blur kernel,
		$\bm{n}$ is the additive white Gaussian noise, and $\ast$ is the discrete 2D convolution. Let $\mathbf{B}\in\mathbb{R}^{L\times K}$ be a subsampled version of the identity matrix 
		and $\mathbf{C}\in\mathbb{R}^{L\times N}$ be the matrix which performs an $N$ dimensional reconstruction of $\mathbf x$ (using the basis vectors of the inverse Haar transform). Let $\mathbf{y} = \text{vec}(\bm{y})$, 
		$\mathbf{x} = \text{vec}(\bm{x}) \in \mathbb{R}^{L}$ , and $\mathbf{k} = \text{vec}(\bm{k})$ (zero-padded to length $L$). Consistent with \cite{Ahmed14}, $\mathbf x$ and $\mathbf k$ are assumed to be in subspaces of dimension $N$ and $K$ respectively with $N,K<L$:
		\begin{align}
			\mathbf{k} = \mathbf{Bh},\quad \mathbf{h}\in \mathbb{R}^K ~~\text{ and } ~~ \mathbf{x} = \mathbf{Cm},\quad \mathbf{m}\in \mathbb{R}^N.
			\label{eqn:dic}
		\end{align}
		By taking 2D discrete Fourier transform on both
		$\bm{x}$ and $\bm{k}$, we can rewrite (\ref{problem}) in the Fourier domain as:
		\begin{equation}
			\hat{\mathrm{y}}_l = \langle \hat{\mathbf{c}}_l, \mathbf{m}\rangle \langle \mathbf{h}, \hat{\mathbf{b}}_l\rangle = \text{Tr}(\mathbf{A}^H_l (\mathbf{h}\mathbf{m}^T)) \quad l=1,\dots,L,
			\label{eqn:lin_op}
		\end{equation}
		where $\mathbf{A}_l = \hat{\mathbf{b}}_l \hat{\mathbf{c}}_l^H$ and
		$\widehat{\cdot}$ is the Fourier transform of corresponding components. By lifting $\mathbf{h}$ and $\mathbf{m}$ to a rank-one matrix via $\mathbf{X}_0 := \mathbf{h}\mathbf{m}^T$, we can compactly represent~(\ref{eqn:lin_op}) as $\hat{\mathbf{y}} = \mathcal{A}(\mathbf{X}_0)$, for a linear operator $\mathcal{A}:\mathbb{R}^{K\times
			N}\longrightarrow\mathbb{C}^L$ defined elementwise in~(\ref{eqn:lin_op}). In this way, the problem
		reduces to a linear inverse problem over the non-convex set comprising rank-one matrices. However, directly solving this non-convex problem is NP-hard~\cite{candes_exact_2009},
		and for tractability, the following convex surrogate is solved instead in \cite{Ahmed14}:
		\begin{align}
			\label{eqn:convex}
			\min_{\mathbf{X}\in\mathbb{R}^{K\times N}} \|\mathbf{X}\|_{*} 
			\text{~~~subject to~~~} \quad \hat{\mathbf{y}} = \mathcal{A}(\mathbf{X}),
		\end{align}
		where $\|\bX\|_\ast$ is the nuclear norm of $\bX$ defined as the sum of
		the singular values of $\bX$.
		It was proven in~\cite{Ahmed14} that, under certain conditions (most
		notably $\max(K, N)\ll L$), $\mathbf{h}$ and $\mathbf{m}$ can be exactly
		recovered (up to scaling) by a singular value decomposition of the solution to~(\ref{eqn:convex}).
		\begin{figure}[t!]
			\begin{center}
				\includegraphics[width=53 mm]{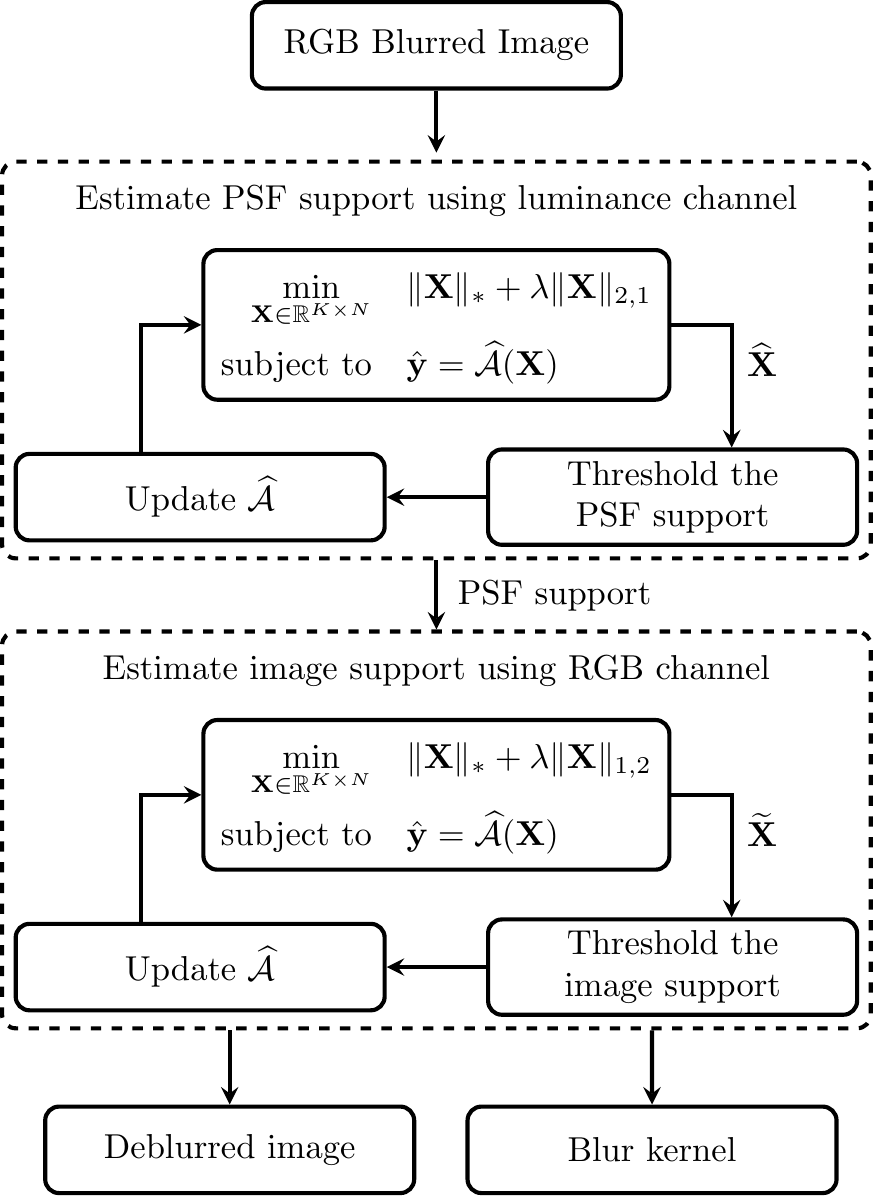}
				\vspace{-.1cm}
				\caption{Flowchart of Blind image Deconvolution using Row-Column Sparsity.}
				\vspace{-0.8cm}
				\label{fig:diagram}
			\end{center}
		\end{figure}
		
		Although this approach provides a promising direction in principle, some
		fundamental limitations prevent its applications to realistic image processing problems. In
		particular, in formulating~(\ref{eqn:convex}) a precise \emph{forward model} is
		implicitly assumed, i.e., $\mathcal{A}$ needs to be exactly determined in advance. However,
		such knowledge in turn requires the {\em exact support} \/of both the image and the blur kernel,
		which is not available in practice.\vspace{-3mm}
		
		\subsection{Tractable estimation of blur kernel and image support}
		\label{sec:bdiss}
		We first show how we detect the kernel support. Although kernel support is
		hard to obtain directly, it is common practice to initialize it with a bounding box
		containing all its nonzero coefficients (refer to Fig.\ 2 for a concrete example). Thus we let
		$\mathbf{h}\in\mathbb{R}^{K'}$ in~(\ref{eqn:dic}) be the vectorized bounding
		box and note correspondingly $\mathbf{B}\in\mathbb{R}^{L\times K'}$ with initial $K'\gg K$, i.e., a vector of all ones and form the corresponding lifted matrix $\bX_0\in\mathbb{R}^{K'\times N}$ as discussed in Section~\ref{subsec:review}. Nonetheless, in practice a sensible recovery of $\bX_0$ cannot be expected simply by solving~(\ref{eqn:convex}) due to the high ambient dimensionality of $\bX_0$; however, it turns out
		$\mathbf{h}$ can be well modeled as a sparse vector. This property translates to \emph{row sparsity} in the lifted matrix
		$\mathbf{X}_0$, i.e., a majority of the rows in
		$\mathbf{X}_0$ are zero vectors.
		To capture it, we
		employ the norm $\|\mathbf{X}\|_{2,1}:=\sum_{i=1}^K\|\mathbf{X}(i,:)\|_2$,
		i.e., sum of $\ell_2$-norms of rows in $\mathbf{X}$. As discussed
		in~\cite{liu_robust_2013}, this norm effectively promotes row sparsity.
		Instead of solving~(\ref{eqn:convex}), we solve the following
		convex optimization problem:
		\noindent
		\begin{align}
			\label{kernl}
			\widehat{\mathbf{X}}\gets\arg\min_{\mathbf{X}\in\mathbb{R}^{K\times N}} \|\mathbf{X}\|_{*} + \lambda\|\mathbf{X}\|_{2,1} \text{~s.t.~} \hat{\mathbf{y}} = \widehat{\mathcal{A}}(\mathbf{X}),
		\end{align}
		where $\widehat{\mathcal{A}}:\mathbb{R}^{K'\times N}\longrightarrow\mathbb{C}^L$ is 
		the linear operator similarly defined as~(\ref{eqn:lin_op}), and $\lambda>0$ is the constant parameter balancing the contributions of both terms and is practically determined via cross validation \cite{crossVal,monga2017}.
		
		Despite the fact that merely solving~(\ref{kernl}) is
		insufficient to yield an accurate estimation of the blur kernel, we
		may gather some partial information about it from the solution
		$\widehat{\mathbf{X}}$. Specifically, when certain rows of
		$\widehat{\mathbf{X}}$ are of significantly smaller magnitudes compared to the others, we may
		infer that those rows should not be included in the kernel support. Via hard thresholding\footnote{In our experiments we fix the threshold value to be half of the mean of the $\ell_2$-norms of each row of $\widehat{\bX}$. In most practical setups, usually $KN>2L$ and thus~(\ref{kernl}) and (\ref{image}) are feasible.} on
		the rows of $\widehat{\mathbf{X}}$, we can improve the forward
		model $\hat{\mathbf{y}} = \widehat{\mathcal{A}}(\mathbf{X}_0)$ by reducing $K'$ until convergence to the {\em true} support.
		Fig.~\ref{fig:psf_refine1} visually
		depicts this procedure for motion PSF of $l=10, \theta=30\degree$.
		
		To recover the image support, we simply note that $\mathbf{k}$ and $\mathbf{x}$
		play symmetric roles in our model. Just as initializing the kernel
		support with a bounding box, we aim to initialize the image support
		with a ``container''. Indeed, in Haar wavelet the locations of dominant image coefficients largely stays
		stable even under large motion blurring, and through mildly thresholding the blurred image coefficients, we can achieve moderate dimensionality
		reduction while maintaining most of the support information.
		With such initialization, we solve the following optimization problem:
		\begin{align}
			\label{image}
			\widetilde{\mathbf{X}}\gets\arg\min_{\mathbf{X}\in\mathbb{R}^{K\times N}} \|\mathbf{X}\|_{*} + \lambda\|\mathbf{X}\|_{1, 2} \text{~s.t.~} \hat{\mathbf{y}} = \widehat{\mathcal{A}}(\mathbf{X}),
		\end{align}
		where $\|\mathbf{X}\|_{1,2}:=\sum_{i=1}^N\|\mathbf{X}(:,i)\|_2$
		effectively promotes column sparsity.
		$\widehat{\mathcal{A}}$ is refined iteratively and a flowchart is shown in Fig.~\ref{fig:diagram}.
		
		\noindent\textbf{Remark:} To be consistent while competing with state of the art methods~\cite{Perrone16,Michaeli2014,Krishnan11,Cai12,Xu_2013_CVPR,Ren16}, we consider noise to be relatively small. Extension to dealing with significant amount of noise is possible by relaxing the equality constraint in (\ref{kernl}) and (\ref{image}) to $\|\hat{\mathbf{y}} - \widehat{\mathcal{A}}(\mathbf{X})\|_2\leq\epsilon$. Further, in finding the support of the PSF a hard-thresholding in each step is performed, and reasonable amount of noise will not influence the detected support. Incorporating the inequality constraint $\|\hat{\mathbf{y}} - \widehat{\mathcal{A}}(\mathbf{X})\|_2\leq\epsilon$ into the optimization problem is beyond the scope of this paper and is a direction for future work.\vspace{-0.5cm}

		\subsection{Efficient Optimization Algorithms}\vspace{-0.5mm}
		\label{subsec:optimization}
		We discuss how to solve~(\ref{kernl}) in detail; solution to~(\ref{image})
		follows analogously. It is difficult to directly solve
		(\ref{kernl}) using similar methods as discussed in~\cite{yin_dual_2015,yin_laplacian_2016}, primarily because of its non-smoothness and high
		dimensionality; however, we can equivalently solve an alternative smooth
		optimization problem of dramatically lower dimensionality, as stated
		formally in the following theorem:\vspace{-2mm}
		
		\begin{theorem}
			The minimizer of~(\ref{kernl}) can be found by solving
			\begin{align}
				&\inf_{\bZ_0\in\bR^{K\times r},\bH_0\in\bR^{N\times r},\bw\in\bR^K_+}\|\bZ_0\|_F^2+\|\bH_0\|_F^2\nonumber\\
				&+\sum_{i=1}^K\left(w_i+\lambda^2\frac{\|\be_i^T\bZ_0\bH_0^T\|_2^2}{w_i}\right)
				\quad\text{s.t. }\,\widehat{\by}=\widehat{\mathcal{A}}(\bZ_0\bH_0^T),
				\label{eqn:final_problem1}
			\end{align}
			to obtain solution $\widehat{\bZ}_0,\widehat{\bH}_0$ and letting $\widehat{\bX}=\widehat{\bZ}_0\widehat{\bH}_0^T$ whenever $r\geq\mathrm{rank}(\widehat{\bX})$ for \emph{every} minimizer $\widehat{\bX}$ of~(\ref{kernl})\footnote{A uniform upper bound for $r$ depending only on $K$ and $N$ is derived in~\cite{burer_local_2005}. Through cross-validation we choose $r =4$.}.
			\label{thm:convergence}
		\end{theorem}\vspace{-1mm}
		
		Proof of Theorem~\ref{thm:convergence} is included in the supplementary document. To solve~(\ref{eqn:final_problem1}) we adopt augmented Lagrangian multiplier method. The augmented Lagrangian is \vspace{-0.1cm}
		\begin{align}
			&\cL_{\sigma,\widehat{\mathcal{A}}}(\bZ_0,\bH_0,\bw;\bm{\alpha})=\|\bZ_0\|_F^2+\|\bH_0\|_F^2-2\bm{\alpha}^H(\widehat{\mathcal{A}}(\bZ_0\bH_0^T)-\widehat{\by})\nonumber\\
			&+\sigma\|\widehat{\mathcal{A}}(\bZ_0\bH_0^T)-\widehat{\by}\|_2^2+\sum_{i=1}^K\left(w_i+\lambda^2\frac{\|\be_i^T\bZ_0\bH_0^T\|_2^2}{w_i}\right),\label{eqn:lagrangian}\vspace{-0.25cm}
		\end{align}
		which is smooth and can be efficiently minimized by the limited memory
		Broyden-Fletcher-Goldfarb-Shanno (L-BFGS) method widely applied to
		large scale nonlinear optimization problems \cite{BFGS}. Instead of jointly
		solving for $(\bZ_0, \bH_0, \bw)$, we alternate between solving for
		$(\bZ_0,\bH_0)$ and $\bw$; more details can be found in the supplementary document. The
		algorithm for solving~(\ref{kernl}) is summarized in
		Algorithm~\ref{alg:almK} ($\sigma_0 = 1e3$, $\varepsilon_0 = 1e-4$ and $\rho = 10$, determined for best convergence rate). The convergence analysis of Algorithm~\ref{alg:almK} is available in \cite{convergence}. The algorithm for solving~(\ref{image}) can be analogously derived.\vspace{-2.5mm}

		\begin{figure}[t!]
			\begin{center}\vspace{0cm}
				\includegraphics[width=0.49\textwidth]{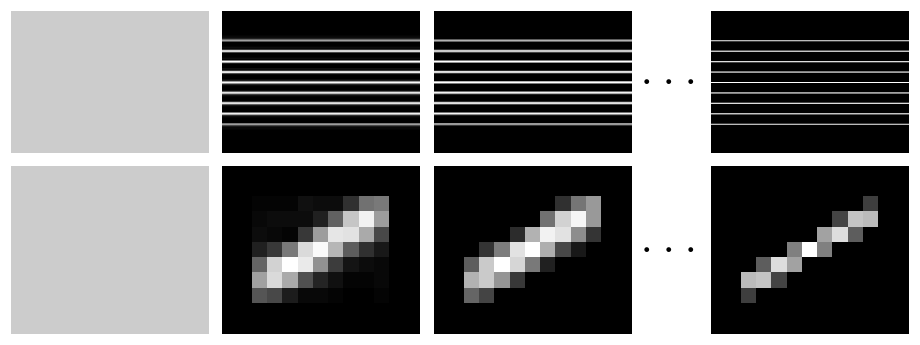}\label{fig:2a}\vspace{-3mm}
				\caption{$\widehat{\mathbf{X}}$ in selected intermediate steps (top row) and the corresponding estimated kernels (bottom row).}\vspace{-5mm}
				\label{fig:psf_refine1}
			\end{center}
		\end{figure}

		\begin{algorithm}[h!]
			\renewcommand{\algorithmicrequire}{\textbf{Input:}}
			\renewcommand{\algorithmicensure}{\textbf{Output:}}
			\caption{\small Augmented Lagrangian Algorithm for Solving~(\ref{kernl})}
			
			\begin{algorithmic}[1]
				\REQUIRE Parameters $\lambda, \rho, \sigma_0, \varepsilon_0\in\mathbb{R}^+$.
				\STATE $\bZ^0\gets\text{uniform kernel}$, $\bH^0\gets\text{blurred image coefficients}$.
				\STATE $\bw^0\gets(\lambda\|\be_i^T\bZ^0\bH^{0T}\|_2 + \varepsilon_0)_{i=1}^K$, $\bm{\alpha}^0\gets\0$.
				\FOR{$k=0$ \TO $I$}
				\STATE $\begin{bmatrix}
				\bZ^{k+1}\\ \bH^{k+1}
				\end{bmatrix}\gets\text{L-BFGS}(\cL_{\sigma_k,\widehat{\mathcal{A}}}(\bZ,\bH,\bw_i;\bm{\alpha}_i))$.
				
				\STATE $\bw^{k+1}\gets(\lambda\|\be_i^T\bZ^{k+1}\bH^{k+1T}\|_2)_{i=1}^K + \frac{\varepsilon_0}{(k+2)^2}$.
				\STATE $\bm{\alpha}^{k+1}\gets\bm{\alpha}^k-\sigma_k(\widehat{\cA}(\bZ^{k+1}\bH^{k+1T})-\widehat{\by})$.
				\STATE $\sigma_{k+1}\gets\rho\sigma_k$.
				\ENDFOR
				\ENSURE $\widehat{\bX}\gets\bZ^{I+1}\bH^{I+1T}$.
			\end{algorithmic}
			\label{alg:almK}
		\end{algorithm}\vspace{-5mm}
		
		\section{Experimental Results}\vspace{-0.5mm}
		\label{sec:Results}
		We evaluate BD-RCS by experimenting on images taken from~\cite{Kohler2012}, with
		motion blur kernels of length from $l=7$ to $l=20$ in different directions, and a Gaussian blur kernel of radius $h=15$ and standard deviation $\sigma = 1.5$ (Fig.~\ref{fig:rowGaus}) to model various realistic scenarios. Note that these PSFs are 2D, although we are using their vectorized form. The test images are included in Fig.~\ref{fig:originals} in the supplementary document, labeled Im1 -- Im5. We compare against seven state-of-the-art methods for blind
		image deconvolution in this paper: TVBD \cite{Perrone16}, BDIRP
		\cite{Michaeli2014}, BDNSM \cite{Krishnan11}, TPKE \cite{Xu2010}, HQMD
		\cite{Shan08}, FBMD \cite{Cai12} and UNL0 \cite{Xu_2013_CVPR}. For quantitative
		evaluations, we compute two widely-used metrics: the Signal-to-Noise Ratio
		(SNR) and the Improvement in Signal-to-Noise (ISNR), both in
		dB. ISNR is given by: $\textrm{ISNR} = 10\times \mathrm{log}_{10}\left(\frac{\|\mathbf{y}-\mathbf{x}\|^2}{\|\mathbf{x}_{\mathrm{rec}}-\mathbf{x}\|^2}\right),$
		%		the following formula:
		%		\begin{equation}
		%			\label{app:eq:eq11}
		%			\textrm{ISNR} = 10\times \mathrm{log}_{10} \left(\frac{\|\mathbf{y}-\mathbf{x}\|^2}{\|\mathbf{x}_{\mathrm{rec}}-\mathbf{x}\|^2}\right),
		%		\end{equation}
		where $\mathbf{y}$ is the blurry image, $\mathbf{x}$ is the original image, and
		$\mathbf{x}_{\mathrm{rec}}$ is the reconstructed image. 
		We initialize the PSF with a bounding box containing all ones. We experimentally found the
		performance of our method robust to variations of the size of the bounding box, so long as it is conservatively 
		selected to enclose the actual support of the blur kernel. In contrast, UNL0,
		BDNSM, HQMD, and BDRIP require the exact size of the kernel to produce the aforementioned results. 
		
		The numerical results for all the aforementioned methods are presented in
		Table~\ref{my-label} for motion PSFs of angle
		$\theta = 30\degree$ and lengths $l = 7, 10, 14$.
		The corresponding original, blurred, and recovered images for $l=10$ are included in the supplementary document as Fig. \ref{fig:row}.
		As indicated by the numerical values, when motion gets larger, all algorithms except for TVBD \cite{Perrone16}, TPKE
		\cite{Xu2010} and BD-RCS fail to plausibly recover the latent images.
		This fact is further supported by the visual results in the supplementary document.
		In particular, for kernels of length $l=10, 14$, and angle $\theta=30\degree$, the
		recovered images for FBMD \cite{Cai12}, UNL0 \cite{Xu_2013_CVPR},  BDNSM
		\cite{Krishnan11}, HQMD \cite{Shan08}, BDRIP \cite{Michaeli2014} suffer
		from either blurring or ringing artifacts. The
		results for BDRIP \cite{Michaeli2014} are typically over-smoothed as also
		reflected by their poor SNR and ISNR values in Table~\ref{my-label}. This can be explained by the wide supports in their reconstructed
		kernels. In contrast, BD-RCS achieves both the highest SNR and ISNR scores in each case, and the sharpest recovered images with the least amount of artifacts. Because only the results for TPKE and TVBD are
		comparable to BD-RCS for $l\geq 10$, we present results only for these methods. The numerical results are presented in Table \ref{my-label1} and one example of this case is shown in Fig. \ref{fig:r4}. Results for the remaining images are included in our supplementary document in \cite{SupDoc}.
		
		As another challenging scenario for image deblurring we present results for an image blurred with a Gaussian kernel. This kind of blur occurs in cameras, microscopes, etc. as a result of defocussing \cite{blindBook}. 
		Fig. \ref{fig:rowGaus} shows visual comparisons of BD-RCS against selected competing methods that are known to be generally applicable to any blur type \cite{Michaeli2014,Krishnan11,Perrone16}.
		
		In presence of more generic motion blur (as in \cite{Levin11}), BD-RCS cannot accurately recover the kernel while adopting Haar wavelet for building $\mathbf{C}$. Haar wavelet is not a very efficient basis for sparse coding natural images. However, BD-RCS is versatile and by adopting a customized basis (e.g. learning a dictionary to build $\mathbf{C}$), it can handle more generic blurs (both in the sense of convergence and PSF recovery, as we have no assumptions on kernel type/shape in our analysis of convergence). This is a topic for future research.
		
		For a $256\times 256$ image the run times\footnote{All methods (except for TPKE and HQMD that are implemented in C and hence excluded for complexity comparisons) were implemented in MATLAB on a computer with an Intel core i7-2600, 3.40 GHz CPU and 8 GB of RAM\@.} are as follows: BDNSM: 68 sec,
		FMBD: 94 sec, UNL0: 296 sec, BDIPR: 300 sec, BD-RCS: 381 sec, and TVBD: 587
		sec. Clearly, BD-RCS enables a more favorable performance-complexity trade-off than competing methods.\vspace{-4mm}

		\begin{table}[t!]
			\centering
			\caption{SNR and ISNR results for image Im3 with motion blur kernel of angle $\theta=30\degree$ and length $l = 7, 10, 14$}\vspace{-0.3cm}
			\label{my-label}
			\scalebox{0.68}{
				\begin{tabular}{|c|c|c|c|c|c|c|c|c|}
					\hline
					\multirow{2}{*}{\textbf{PSF}} & \multicolumn{2}{c|}{\textbf{BD-RCS}}  & \multicolumn{2}{c|}{\textbf{TVBD}} & \multicolumn{2}{c|}{\textbf{TPKE}} & \multicolumn{2}{c|}{\textbf{BDNSM}} \\ \cline{2-9} 
					& SNR& ISNR& SNR& ISNR& SNR& ISNR& SNR& ISNR\\ \hline
					\textbf{7}& \textbf{22.96}& \textbf{4.28}& 22.44& 3.26& 20.18& 1.50& 20.19& 1.51\\ \hline
					\textbf{10}& \textbf{22.01}& \textbf{4.93}& 20.31& 3.23& 19.01& 1.94& 19.17& 2.09\\ \hline
					\textbf{14}& \textbf{21.08}& \textbf{5.54}& 17.02& 1.48& 16.93& 1.39& 15.14& -0.40\\ \hline
					& \multicolumn{2}{c|}{\textbf{NUL0}} & \multicolumn{2}{c|}{\textbf{HQMD}} & \multicolumn{2}{c|}{\textbf{FMBD}}& \multicolumn{2}{c|}{\textbf{BDIPR}}  \\ \hline
					\textbf{7}& 19.59& 0.91& 19.96& 1.28& 18.32& -0.36& 17.84& -0.84\\ \hline
					\textbf{10}& 17.60& 0.52& 18.01& 0.92& 16.93& -0.20& 15.40& -1.67\\ \hline
					\textbf{14}& 16.44& 0.89& 16.21& 0.67& 15.19& -0.35& 13.75& -1.80\\ \hline
				\end{tabular}}\vspace{-.5cm}
			\end{table}

			\begin{table}[t!]
				\centering
				\caption{SNR and ISNR results for motion blur kernel of angle $\theta=30\degree$ and length $l = 15, 20$
					for images Im1 -- Im4}
				\vspace{-0.3cm}
				\label{my-label1}
				\scalebox{0.7}{
					\begin{tabular}{|c|c|c||c|c|c|c|c|c|}
						\hline
						\multirow{2}{*}{\textbf{Image}} & \multicolumn{2}{c||}{\textbf{Non-blind}}  & \multicolumn{2}{c|}{\textbf{BD-RCS}} & \multicolumn{2}{c|}{\textbf{TVBD}} & \multicolumn{2}{c|}{\textbf{TPKE}} \\ \cline{2-9} 
						& SNR& ISNR& SNR& ISNR& SNR& ISNR& SNR& ISNR\\ \hline
						\textbf{Im1} ($l=15$)& 24.43& 12.49& \textbf{17.66}& \textbf{5.72}& 14.00& 2.06& 15.91& 3.97\\ \hline
						\textbf{Im2} ($l=20$)& 23.30& 14.09& \textbf{16.10}& \textbf{6.87}& 11.88& 2.65& 12.51& 3.29\\ \hline
						\textbf{Im3} ($l=15$)& 27.48& 13.54& \textbf{21.06}& \textbf{7.14}& 14.06& 0.14& 16.44& 2.52\\ \hline
						\textbf{Im4} ($l=20$)& 23.36& 12.84& \textbf{16.95}& \textbf{6.40}& 15.34& 4.79& 12.19& 1.64\\ \hline
					\end{tabular}}\vspace{-0.3cm}
				\end{table}

				%	\clearpage
				
				\begin{figure}[ht!]
					\begin{center}
						\includegraphics[width=88mm]{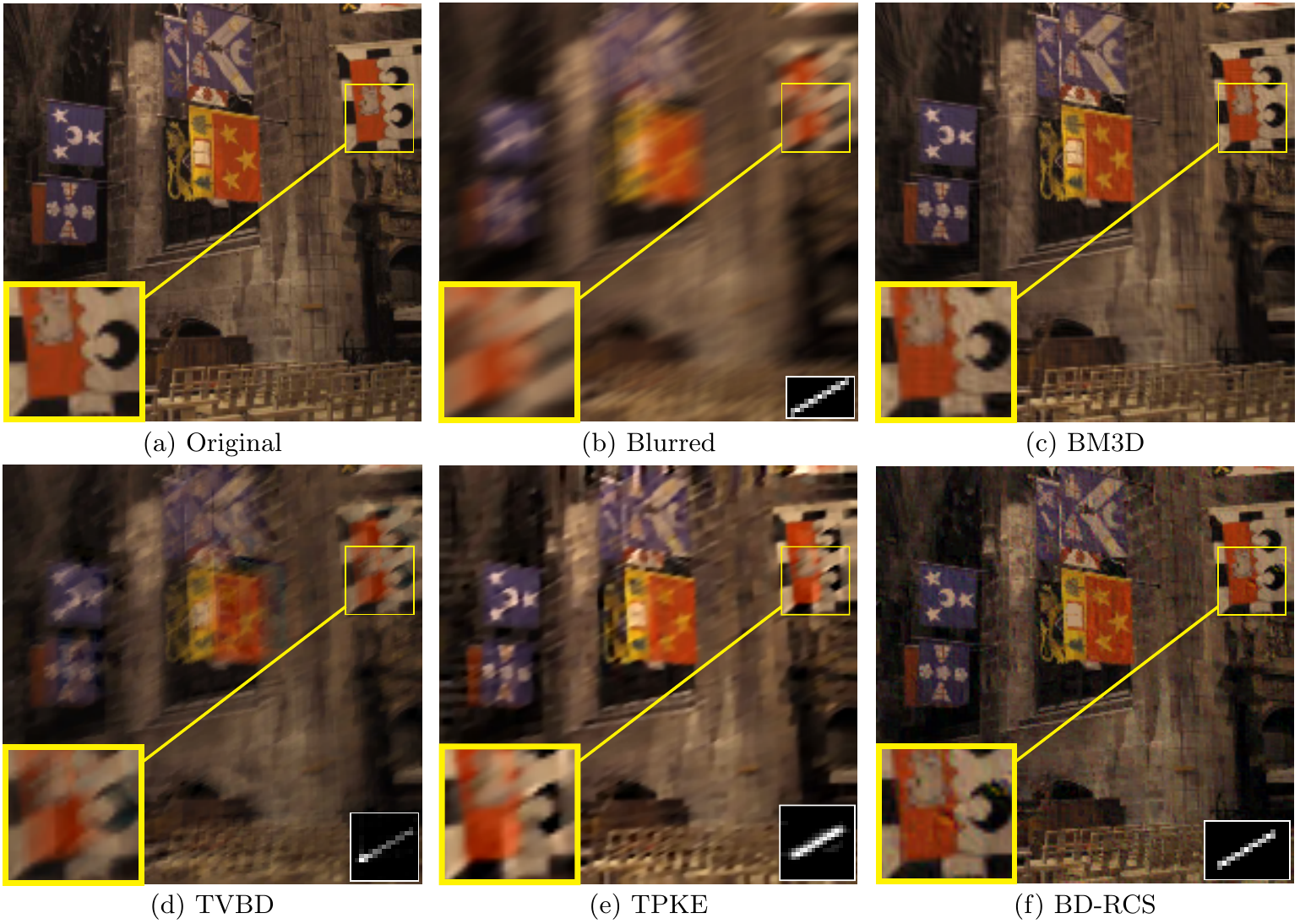}\vspace{-.35cm}
						\caption{a) Original image (Im1), b) Original PSF ($l=20, \theta=30\degree$) and blurred image, c) BM3D non-blind deblurring \cite{BM3Dblur}; SNR = 23.36 dB, ISNR = 12.84 dB, d) TVBD; SNR = 15.34 dB, ISNR = 4.79, e) TPKE; SNR = 12.19 dB, ISNR = 1.64 dB, f) BD-RCS; SNR = 16.95 dB, ISNR = 6.40 dB.\vspace{-5mm}}\label{fig:r4}
					\end{center}\vspace{-0.15cm}
				\end{figure}
				
				\begin{figure}[ht!]
					\begin{center}
						\includegraphics[width=88mm]{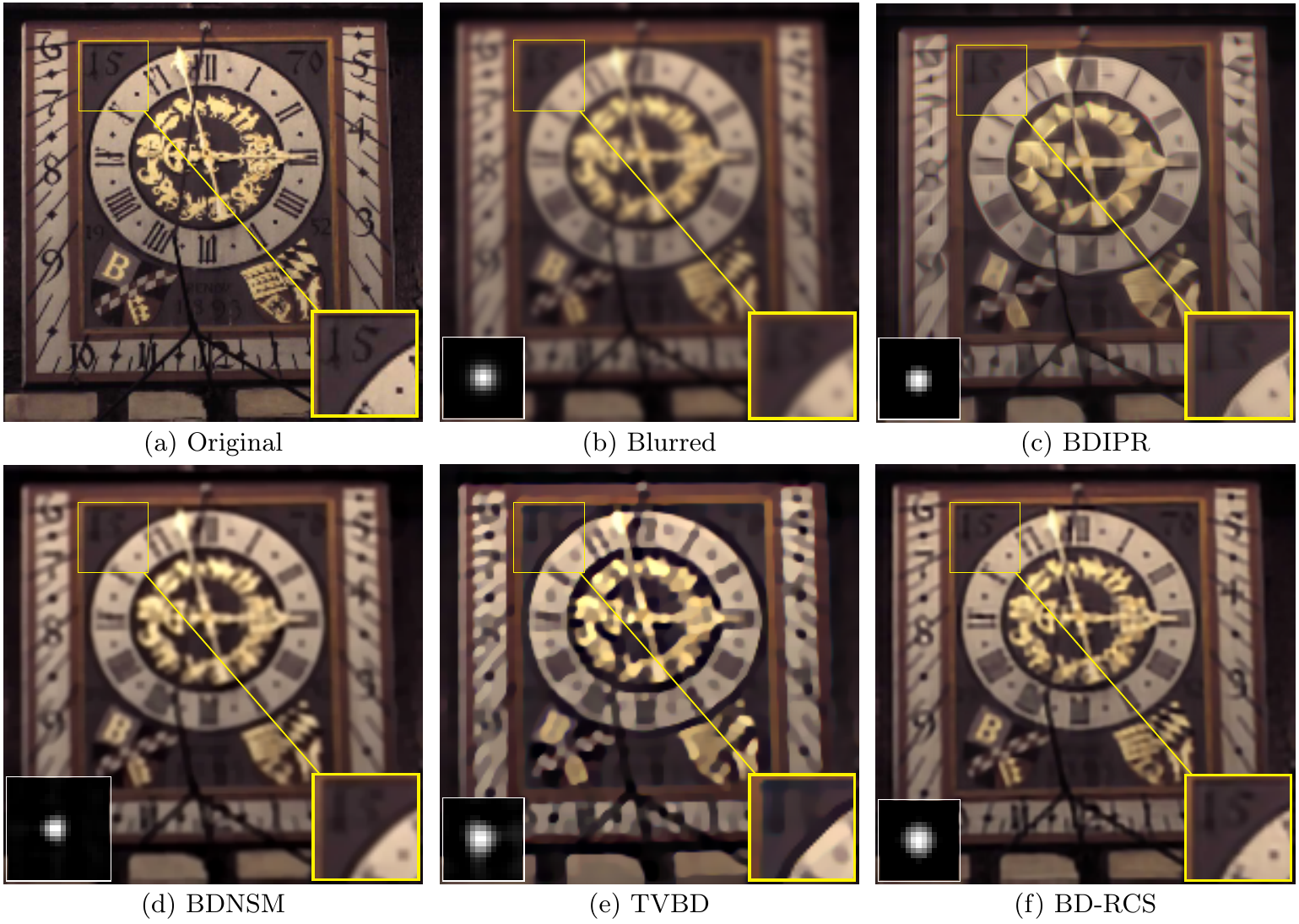}\vspace{-.35cm}
						\caption{a) Original image (Im2), b) Gaussian PSF ($h = 15$ and $\sigma = 1.5$) and blurred image, c) BDIRP \cite{Michaeli2014}; SNR = 13.84 dB, d) BDNSM \cite{Krishnan11}; SNR=13.53 dB, e) TVBD \cite{Perrone16}; SNR = 12.22 dB, and f) BD-RCS; SNR = 16.68 dB.\vspace{-6mm}} \label{fig:rowGaus}
					\end{center}\vspace{-0.4cm}
				\end{figure}
				
				\section{Conclusion}\vspace{-1mm}
				\label{sec:refs}
				We present a novel blind image deblurring method based on structured sparse representations. Our central contribution is to develop a practical realization of a principled rank minimization framework for deconvolution by setting up practical and tractable sparsity constrained optimization problems, enabling  accurate estimation of the blur kernel and image support. The proposed BD-RCS achieves a favorable cost-quality trade-off against state of the art approaches. Our work first estimates the blur kernel support followed by solving for the deblurred image. Algorithmic extensions could include performing these two steps in an alternating fashion until convergence.

				%Future work:
				
				%\begin{spacing}{0.9}%1
				%    \setlength{\bibsep}{0pt}%1pt
				\bibliographystyle{IEEEbib}
				\bibliography{PhdReferences}
				%\end{spacing}

				%%%%%%%%%%%%%%%%%%%%%%%%%%%%%%%%%%%%%%%%%%%%%%%%%%%%%%%%
				\newpage
				
%				\setcounter{page}{1}
				%\centering
				%\appendix[Solution of Optimization Using Duality]
				\section{Solution of Optimization Using Duality}
				\label{sec:Proof}
				
				\emph{Proof of Theorem \ref{thm:convergence}}: The basic idea is technically similar
				to~\cite{recht_guaranteed_2010}, although the derivations are more involved due
				to the addition of $\|\mathbf{X}\|_{2,1}$. First we need the following
				proposition:
				\begin{proposition}
					The dual norm of $\|\bX\|_{2, 1}$ is $\|\bX\|_{2,\infty}:=\max_{1\leq i\leq K}\|\be_i^T\bX\|_2$.
				\end{proposition}
				\begin{proof}
					$\forall\bY\in\bR^{K\times N}:\|\bY\|_{2, 1}=1,$
					\begin{align}
						\langle\bX,\bY\rangle_F&=\sum_{i=1}^K\langle\bX_{i,:},\bY_{i,:}\rangle\overset{\mathrm{i}}{\leq}\sum_{i=1}^K\|\bX_{i,:}\|_2\|\bY_{i,:}\|_2\nonumber\\
						&\leq\|\bX\|_{2,\infty}\|\bY\|_{2, 1}=\|\bX\|_{2,\infty}.
						\label{proof1}
					\end{align}
					where in i we invoke Cauchy-Schwartz inequality. From the definition of dual norm
					\begin{align}
						\left[\|\bX\|_{2,1}\right]_D=\sup_{\|\bY\|_{2,1}=1}\langle\bX,\bY\rangle_F\leq\|\bX\|_{2,\infty}
					\end{align}
					where $[\|\cdot\|]_D$ denotes the dual norm of $\|\cdot\|$. On the other hand, $\forall\bX\in\bR^{K\times N}$, pick $i_0\in\arg\max_i\|\bX_{i,:}\|_2$, define $\widetilde{\bY}$ as
					\begin{align}
						\widetilde{\bY}_{i,:}=
						\begin{cases}
							\frac{\bX_{i_0,:}}{\|\bX_{i_0,:}\|_2} & i=i_0,\\
							\mathbf{0} & i\neq i_0,
						\end{cases}
					\end{align}
					to establish $[\|\bX\|_{2,1}]_D=\sup_{\|\bY\|_{2,1}=1}\langle\bX,\bY\rangle_F\geq\langle\bX,\widetilde{\bY}\rangle_F=\|\bX\|_{2,\infty}$.
				\end{proof}
				
				Also recall the dual norm of $\|\bX\|_\ast$ is $\|\mathbf{X}\|_2$, i.e., the maximum singular value of $\mathbf{X}$. Using the definition of dual norm, we may reexpress $\|\bX\|_\ast+\lambda\|\bX\|_{2,1}$ as solution to the following optimization problem
				\begin{align}
					&\max_{\bY_1,\bY_2\in\bR^{K\times N}}\langle\bX,\bY_1\rangle_F+\lambda\langle\bX,\bY_2\rangle_F\nonumber\\
					&\text{subject to~}\|\bY_1\|_2\leq 1,~\|\bY_2\|_{2,\infty}\leq 1,
				\end{align}
				and rewrite~(\ref{kernl}) as the following minimax problem:
				\begin{align}
					&\min_{\bX\in\mathbb{R}^{K\times N}}\max_{\bY_1,\bY_2\in\mathbb{R}^{K\times N}}\langle\bX,\bY_1\rangle_F+\lambda\langle\bX,\bY_2\rangle_F\nonumber\\
					&\text{subject to~}\|\bY_1\|_2\leq 1,\|\bY_2\|_{2,\infty}\leq1, \widehat{\mathbf{y}}=\mathcal{A}(\bX).
				\end{align}
				To simplify this, let us first focus on the maximization problem:
				\begin{align}
					(P_1):\max_{\bY_1,\bY_2\in\bR^{K\times N}}&\left\langle
					\begin{pmatrix}
						\bX\\
						\lambda\bX
					\end{pmatrix},
					\begin{pmatrix}
						\bY_1\\
						\bY_2
					\end{pmatrix}\right\rangle_F\nonumber\\
					\text{subject to }&\|\bY_1\|_2\leq 1,~\|\bY_2\|_{2,\infty}\leq 1.
				\end{align}
				
				$(P_1)$ is equivalent to the following semi-definte programming (SDP) problem
				\begin{align}
					(P_2):&\max_{\bY_1,\bY_2\in\bR^{K\times N}}\left\langle
					\begin{pmatrix}
						\bX\\\lambda\bX
					\end{pmatrix},
					\begin{pmatrix}
						\bY_1\\\bY_2
					\end{pmatrix}\right\rangle_F\nonumber\\
					\text{subject to ~~}&
					\tiny{\begin{bmatrix}
							\bI_K & \bY_1 & & & & & \\
							\bY_1^T & \bI_N & & & & \bigzero &\\
							& & 1 & \be_1^T\bY_2 & & &\\
							& & \bY^T_2\be_1 & \bI_N & & & \\
							& & & &\ddots & & \\
							& & \bigzero & & & 1 & \be_K^T\bY_2\\
							& & & & & \bY_2^T\be_K & \bI_N
						\end{bmatrix}\succeq\mathbf{0}},
				\end{align}
				and the following dual program
				\begin{align}
					(P_2^\ast):&\min_{\overline{\bX}, \bW_0, \{w_i\}_{i=1}^K, \{\bV_i\}_{i=0}^K}\tr(\overline{\bX})\nonumber\\
					\text{s.t.~}&
					\tiny{\overline{\bX}=\begin{bmatrix}
							\bW_0 & \bX & & & & & \\
							\bX^T & \bV_0 & & & & \bigast &\\
							& & w_1 & \lambda\be_1^T\bX & & &\\
							& & \lambda\bX^T\be_1 & \bV_1 & & & \\
							& & & &\ddots & & \\
							& &\bigast & & & w_K & \lambda\be_K^T\bX\\
							& & & & & \lambda\bX^T\be_K & \bV_K
						\end{bmatrix}\succeq\mathbf{0}},
					\label{eqn:dual}
				\end{align}
				where
				$\overline{\bX}\in\bR^{2K+(K+1)N\times2K+(K+1)N},\bW_0\in\bR^{K\times
					K},\{w_i\in\bR\}_{i=1}^K,\{\bV_i\in\bR^{N\times N}\}_{i=0}^K$, $\bI_k$
				is the $K\times K$ identity matrix and $\be_i$ denotes the $i$-th
				standard basis vector.
				
				We apply the \emph{Burer-Monteiro Factorization}~\cite{burer_nonlinear_2003} to eliminate the positive semi-definite constraint and reduce the dimensionality.
				\begin{align}
					\overline{\bX}=
					\begin{bmatrix}
						\bZ_0\\\bH_0\\\bz_1^T\\\bH_1\\\vdots\\\bz_K^T\\\bH_K
					\end{bmatrix}
					\begin{bmatrix}
						\bZ_0\\\bH_0\\\bz_1^T\\\bH_1\\\vdots\\\bz_K^T\\\bH_K
					\end{bmatrix}^T{\text{where }
					\bZ_0\in\bR^{K\times r},\bH_i\in\bR^{N\times r}, \bz_i\in\bR^r},
				\label{eqn:factor}
			\end{align}
			for some $r\geq \mathrm{rank}(\overline{\bX})$. Equating two expressions of $\overline{\bX}$ in~(\ref{eqn:dual}) and~(\ref{eqn:factor}) yields
			\begin{align}
				&\bZ_0\bZ_0^T=\bW_0,~\bH_0\bH_0^T=\bV_0,\nonumber\\
				&\|\bz_i\|_2^2=w_i,~\bH_i\bH_i^T=\bV_i,~\bZ_0\bH_0^T=\bX,\nonumber\\
				&\bz_i^T\bH_i^T=\lambda\be_i^T\bX=\lambda\be_i^T\bZ_0\bH_0^T,~i=1,2, \dots, K.
				\label{eqn:equality}
			\end{align}
			Therefore, we can rewrite $\tr(\overline{\bX})$ as:
			\begin{align}
				\tr(\widetilde{\bX})&=\tr(\bW_0)+\tr(\bV_0)+\sum_{i=1}^Tw_i+\sum_{i=1}^K\tr(\bV_i)\nonumber\\
				&=\|\bZ_0\|_F^2+\|\bH_0\|_F^2+\sum_{i=1}^K w_i+\sum_{i=1}^K\|\bH_i\|_F^2.
				\label{eqn:trace}
			\end{align}
			Plugging~(\ref{eqn:equality}) and~(\ref{eqn:trace}) back into $(P_1)$, we get
			\begin{align}
				(P_1^\ast):&\min_{\bZ_0\in\bR^{K\times r},\{\bH_i\in\bR^{N\times r}\}_{i=0}^K,\{\bz_i\in\bR^r\}_{i=1}^K}\|\bZ_0\|_F^2\nonumber\\
				&+\|\bH_0\|_F^2+\sum_{i=1}^K w_i+\sum_{i=1}^K\|\bH_i\|_F^2\nonumber\\
				&\text{subject to }~\bz_i^T\bH_i^T=\lambda\be_i^T\bZ_0\bH_0^T,~ \widehat{\by}=\cA(\bZ_0\bH_0^T).
			\end{align}
			Note the problem 
			\begin{equation}
				\min_{\bH_i\in\bR^{N\times r}}\|\bH_i\|_F^2\quad\text{subject to}\quad\bz_i^T\bH_i^T=\lambda\be_i^T\bZ_0\bH_0^T
			\end{equation}
			has closed form solution
			\begin{align}
				\bH_i^\ast=
				\begin{cases}
					\frac{\lambda\bH_0\bZ_0^T\be_i\bz_i^T}{\|\bz_i\|_2^2} & \bz_i\neq\0\\
					\0 & \bz_i=\0 \text{ and } \be_i^T\bZ_0\bH^T_0 = \0 \\
					\text{infeasible} & \text{otherwise}.
				\end{cases}
			\end{align}
			Then we can replace $\bH_i, i=1, 2,\dots, K$ by its closed form solution in $(P_1^\ast)$ to get (note $w_i=\|\bz_i\|_2^2$)
			\begin{align}
				(P_3):\inf_{\bZ_0\in\bR^{K\times r},\bH_0\in\bR^{N\times r},\bw\in\bR^K_+}\|\bZ_0\|_F^2+\|\bH_0\|_F^2\nonumber\\
				+\sum_{i=1}^K\left(w_i+\lambda^2\frac{\|\be_i^T\bZ_0\bH_0^T\|_2^2}{w_i}\right)
				\text{ ~~~~s.t. }\widehat{\by}=\cA(\bZ_0\bH_0^T).
				\label{eqn:final_problem}
			\end{align}
			Finally, since $\bX=\bZ_0\bH_0^T$, $\mathrm{rank}(\bX)=\mathrm{rank}(\bZ_0\bH_0^T)=\mathrm{rank}(\bZ_0)=\mathrm{rank}(\bH_0)$, and thus we can find $\widehat{\bX}$ as long as $r\geq\mathrm{rank}(\overline{\bX})$.
			
			\qed
			
			As discussed in~\ref{subsec:optimization}, we solve~(\ref{eqn:lagrangian}) using the BFGS method; for higher efficiency, we analytically derive the gradients of $\bZ_0$ and $\bH_0$ as follows
			\begin{align}
				\begin{bmatrix}
					\nabla_\bZ\cL_{\sigma,\cA}\\\nabla_\bH\cL_{\sigma,\cA}
				\end{bmatrix}=2\begin{bmatrix}
				\bZ_0-\cA^\ast(\widehat{\bm{\alpha}})\bH_0+\lambda^2\diag(\bw)^{-1}\bZ_0\bH_0^T\bH_0\\
				\bH_0-\cA^\ast(\widehat{\bm{\alpha}})^T\bZ_0+\lambda^2\bH_0\bZ_0^T\diag(\bw)^{-1}\bZ_0
			\end{bmatrix},
		\end{align}
		where $\widehat{\bm{\alpha}}=\bm{\alpha}-\sigma(\cA(\bZ_0\bH_0^T)-\widehat{\by})$, and $\diag(\bw)$ is the diagonal matrix with its diagonal equal to $\bw$. In this optimization, $\alpha$ is the Lagrange multiplier, $\sigma$ is the penalty parameter where $\sigma_0 = 1e3$, and $\rho = 10$ is the scale factor. Through cross-validation \cite{monga2017} we observed that $r\geq 4$ yields almost identical results regardless of different image sets and initializations; therefore, we choose $r = 4$. Fig.~\ref{fig:psf_refine} presents the estimated kernels per iteration.\vspace{0cm}
		
		\section{Extra Experimental Results}
		As we mentioned in Sec. \ref{sec:Results}, we further increase the kernel length ($l\geq 15$) and included the
		corresponding numerical results only for TPKE, TVBD, and BD-RCS in Table~\ref{my-label1} for images Im1 -- Im4, respectively.
		The visual results for Im2 -- Im4 can be found in Fig.~\ref{fig:r1} -- Fig.~\ref{fig:3}, respectively. For Im4, the original image and its blurred version are included in Fig.~\ref{fig:3a} and
		Fig.~\ref{fig:3b}, respectively. The motion blur kernel is of $l = 20, \theta =
		30\degree$.  Fig.~\ref{fig:3c} includes the result of a well-known
		non-blind image deconvolution algorithm, the BM3D deblurring
		algorithm~\cite{BM3Dblur}, as an approximation to the performance
		upper-bound for any blind image deconvolution algorithms.
		Fig.~\ref{fig:r1a}-~\ref{fig:r1c} are the recovered images from BD-RCS,
		TVBD and TPKE, respectively. To compare the spatial details of the images,
		we magnify a common part of these images and place it in the bottom right
		portion of each image. We also place the reconstructed blur kernel in the
		bottom left portion. Clearly, the TPKE algorithm usually produces ringing artifacts (Gibbs phenomenon)
		in the recovered images particularly around edges as it works in the Fourier domain. As for the kernel, some
		residuals outside the true support can be observed, although the direction
		appears correct.  This contributes to the blurring artifacts in the
		recovered image; in particular, the pattern in the magnified portion is
		distorted. The TVBD algorithm recovers the blur kernel more accurately;
		however, details of the recovered image appear over-smoothed.
		In contrast, the BD-RCS algorithm accurately recovers both the
		blur kernel and the latent image, with visual quality comparable to
		BM3D~\cite{BM3Dblur}. In particular, the details in the magnified part
		appear almost as sharp and clear as the original image. Furthermore, the
		reconstructed blur kernel is also close to the original one. Indeed, their
		supports are exactly the same.\vspace{0mm}

		In this paper, we mostly focused on linear motion blur and to further corroborate the versatility of BD-RCS, we also added deblurring results for an image blurred with a Gaussian blur kernel which is represented in Fig.~\ref{fig:rowGaus} for BD-RCS and selected competing methods that work for general type of blur kernels \cite{Michaeli2014,Krishnan11,Perrone16}.\vspace{-0.4cm}
		\begin{figure}[h!]
			\begin{center}
				\subfloat{\includegraphics[width=20mm]{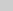}\label{fig:2a}}\,\!
				\subfloat{\includegraphics[width=20mm]{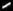}}\,\!
				\subfloat{\includegraphics[width=20mm]{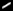}}\,\!
				\subfloat{\includegraphics[width=20mm]{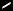}} \\\vspace{-0.1cm}
				\subfloat{\includegraphics[width=20mm]{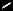}}\,\!
				\subfloat{\includegraphics[width=20mm]{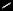}}\,\!
				\subfloat{\includegraphics[width=20mm]{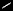}}\,\!
				\subfloat{\includegraphics[width=20mm]{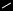}}\,\!\vspace{-0.1cm}
				\caption{Intermediate results of PSF (top left to bottom right).}\vspace{-0cm}
				\label{fig:psf_refine}
			\end{center}
		\end{figure}\vspace{-.5cm}
		\begin{figure}[h!]
			\begin{center}
				\subfloat[Im1]{\includegraphics[width=25mm]{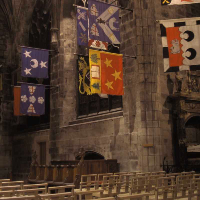} \label{fig:G1}}
				\subfloat[Im2]{\includegraphics[width=25mm]{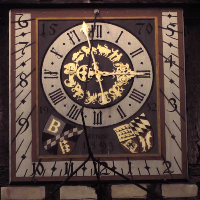} \label{fig:G2}}
				\subfloat[Im3]{\includegraphics[width=25mm]{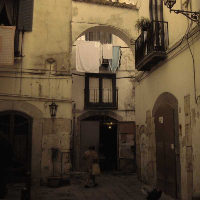} \label{fig:G3}}\\\vspace{-.1cm}
				\subfloat[Im4]{\includegraphics[width=25mm]{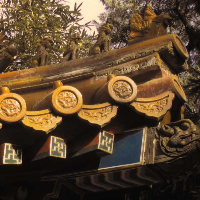} \label{fig:G4}}
				\subfloat[Im5]{\includegraphics[width=25mm]{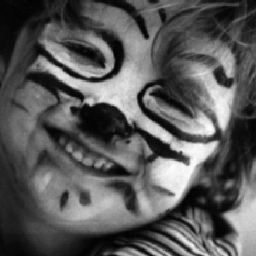} \label{fig:test}}\vspace{-.1cm}
				\caption{Images used in our experiments: Im1 to Im5 from left to right}
				\label{fig:originals}
			\end{center}
		\end{figure}\vspace{-.5cm}

		\begin{figure}[h!]\vspace{-.2cm}
			\begin{center}
				\subfloat[]{\includegraphics[width=28.5mm]{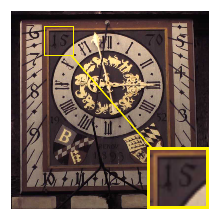} \label{fig:4a}}
				\subfloat[]{\includegraphics[width=28.5mm]{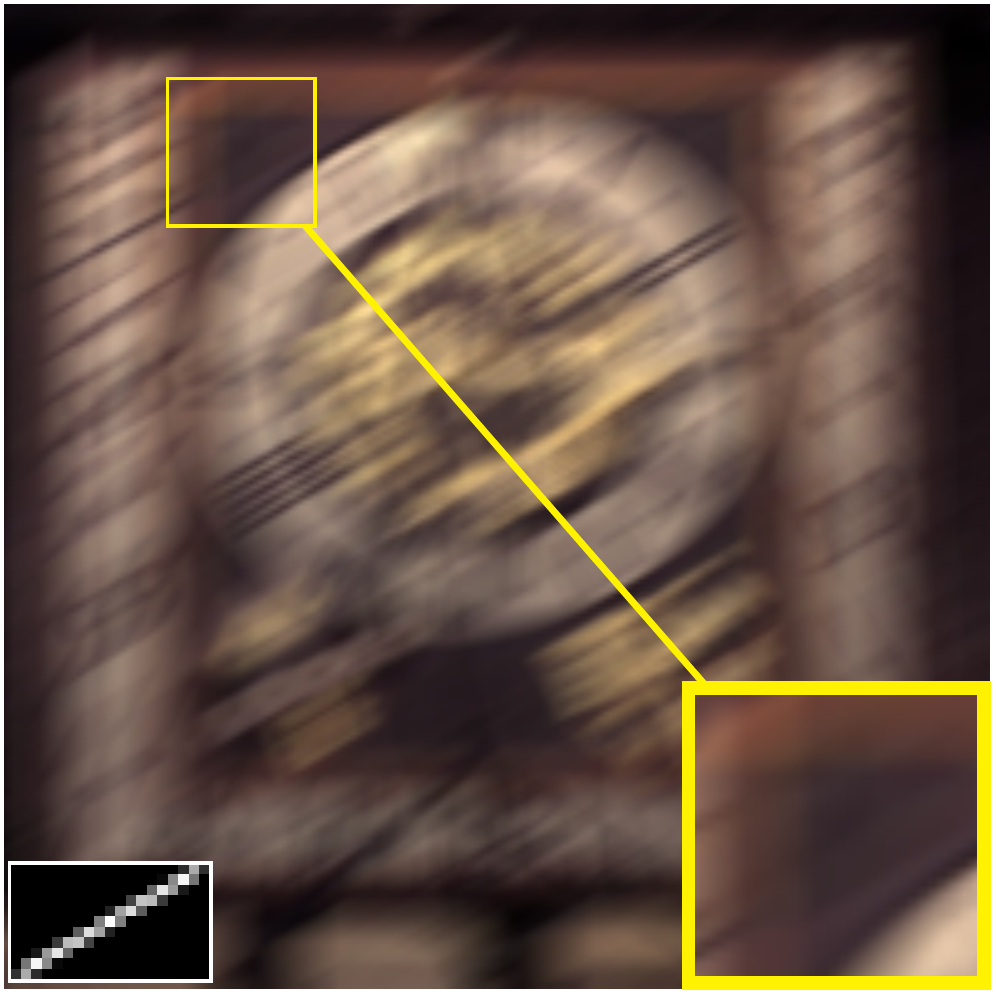} \label{fig:4b}}
				\subfloat[]{\includegraphics[width=28.5mm]{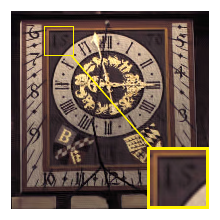} \label{fig:4c}}\\\vspace{-0.2cm}
				\subfloat[]{\includegraphics[width=28.5mm]{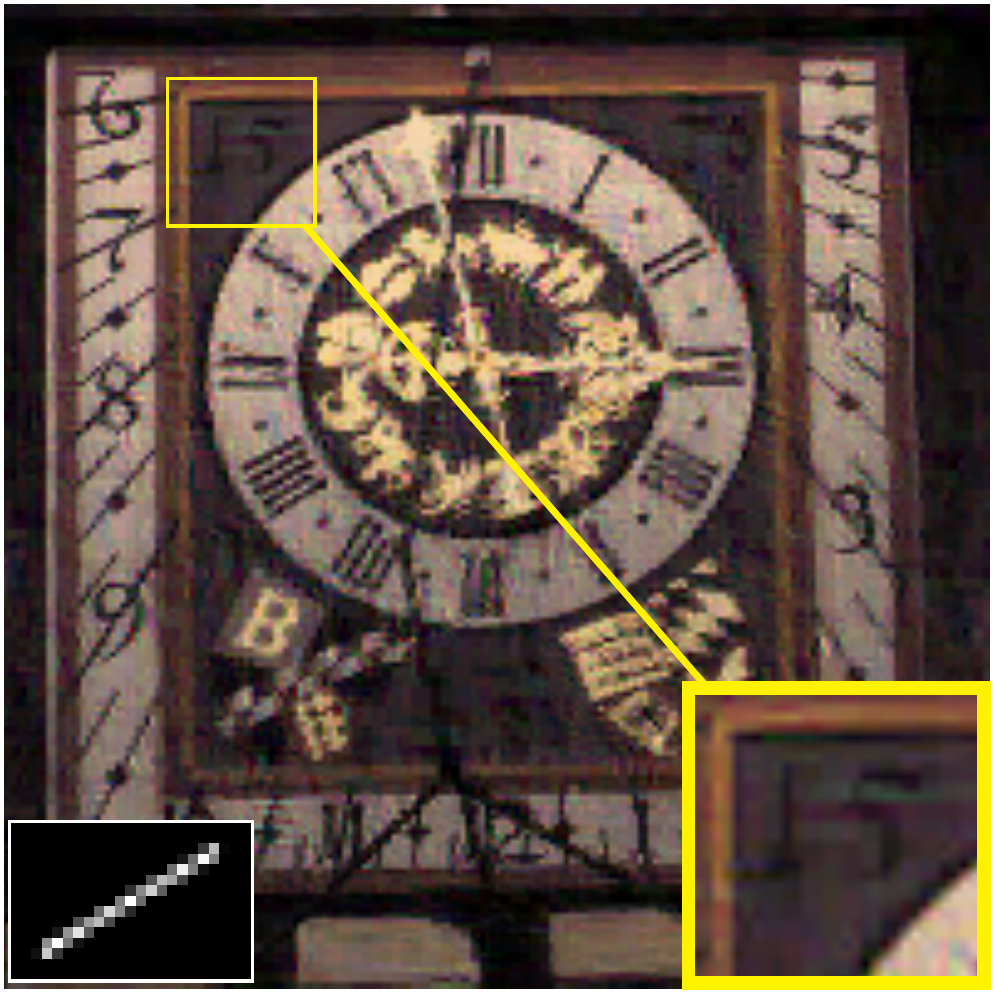} \label{fig:q1a}}
				\subfloat[]{\includegraphics[width=28.5mm]{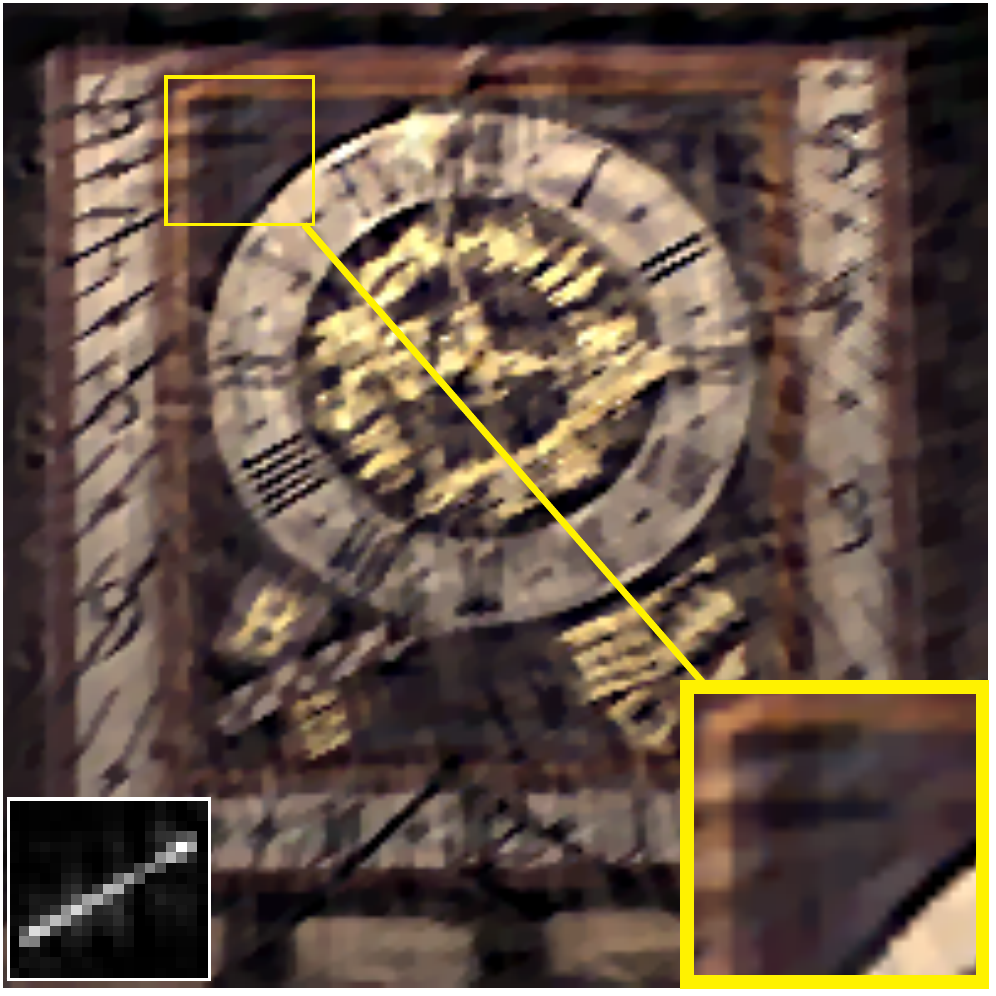} \label{fig:q1b}}
				\subfloat[]{\includegraphics[width=28.5mm]{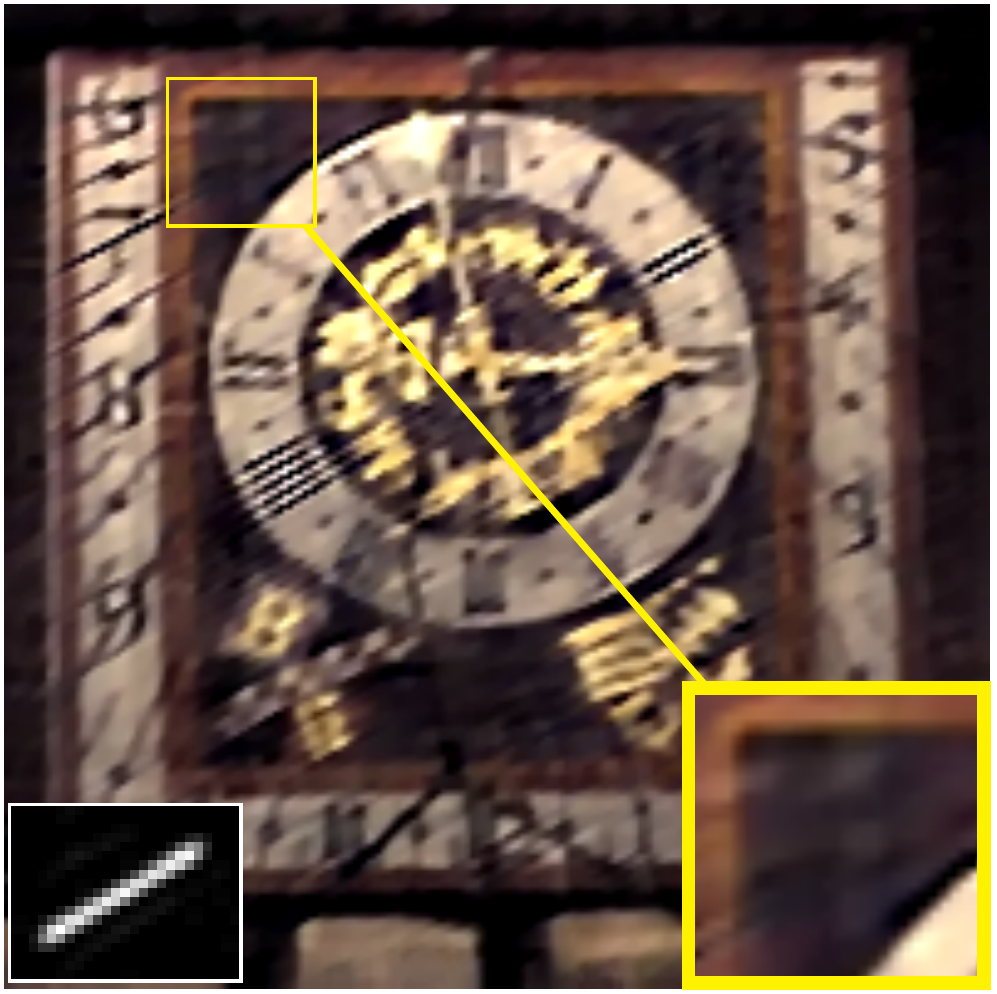} \label{fig:q1c}}\vspace{-0.1cm}
				\caption{a) Original image (Im2), b) Original PSF ($l=20, \theta=30$) and blured image, c) BM3D non-blind deblurring \cite{BM3Dblur}; SNR = 23.30 dB, ISNR = 14.09 dB, d) BD-RCS; SNR = 16.10 dB, ISNR = 6.87 dB, e) TVBD; SNR = 11.88 dB, ISNR = 2.65 dB, f) TPKE; SNR = 12.51 dB, ISNR = 3.29 dB.}\vspace{-0.8cm}
				\label{fig:r1}
			\end{center}
		\end{figure}

		\setcounter{figure}{9}
		\begin{figure*}[!b]
			\begin{center}
				\includegraphics[width=150mm]{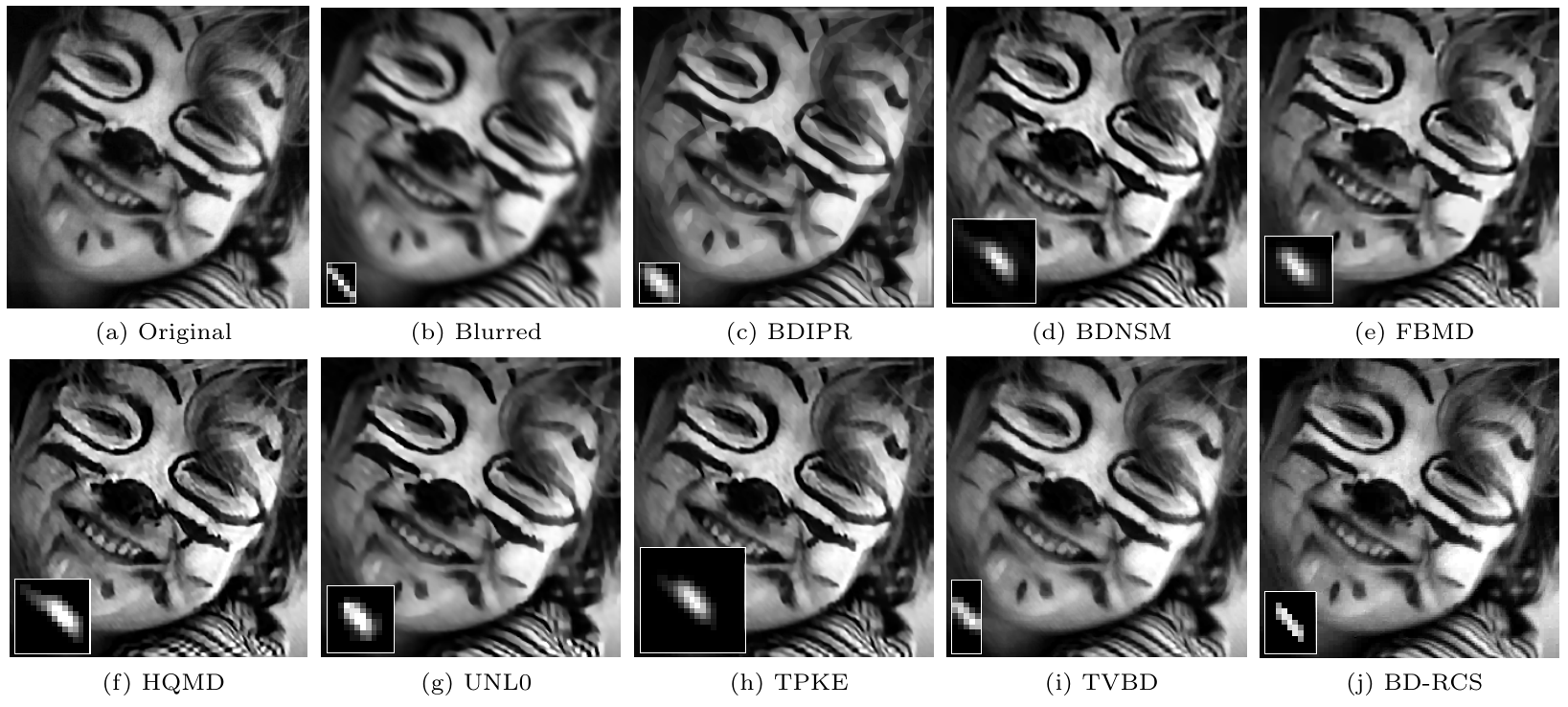}\vspace{-.4cm}
				\caption{Results for small kernel ($l = 7$ and $\theta = 130^\degree$): a) Original image, b) Original PSF ($l=7, \theta=130$) and blured image, c) BDIRP \cite{Michaeli2014}; SNR = 20.88 dB, d) BDNSM \cite{Krishnan11}; SNR=16.29 dB, e) FBMD \cite{Cai12}; SNR=20.65 dB, f) HQMD \cite{Shan08}; SNR=16.72 dB, g) UNL0 \cite{Xu_2013_CVPR}; SNR=17.03, h) TPKE \cite{Xu2010}; SNR = 20.47 dB, i) TVBD \cite{Perrone16}; SNR = 23.79 dB, and j) BD-RCS; SNR = 22.86 dB.} \label{fig:row}
			\end{center}\vspace{-.2cm}
		\end{figure*}
		
		\begin{figure*}[!b]
			\begin{center}
				\includegraphics[width=150mm]{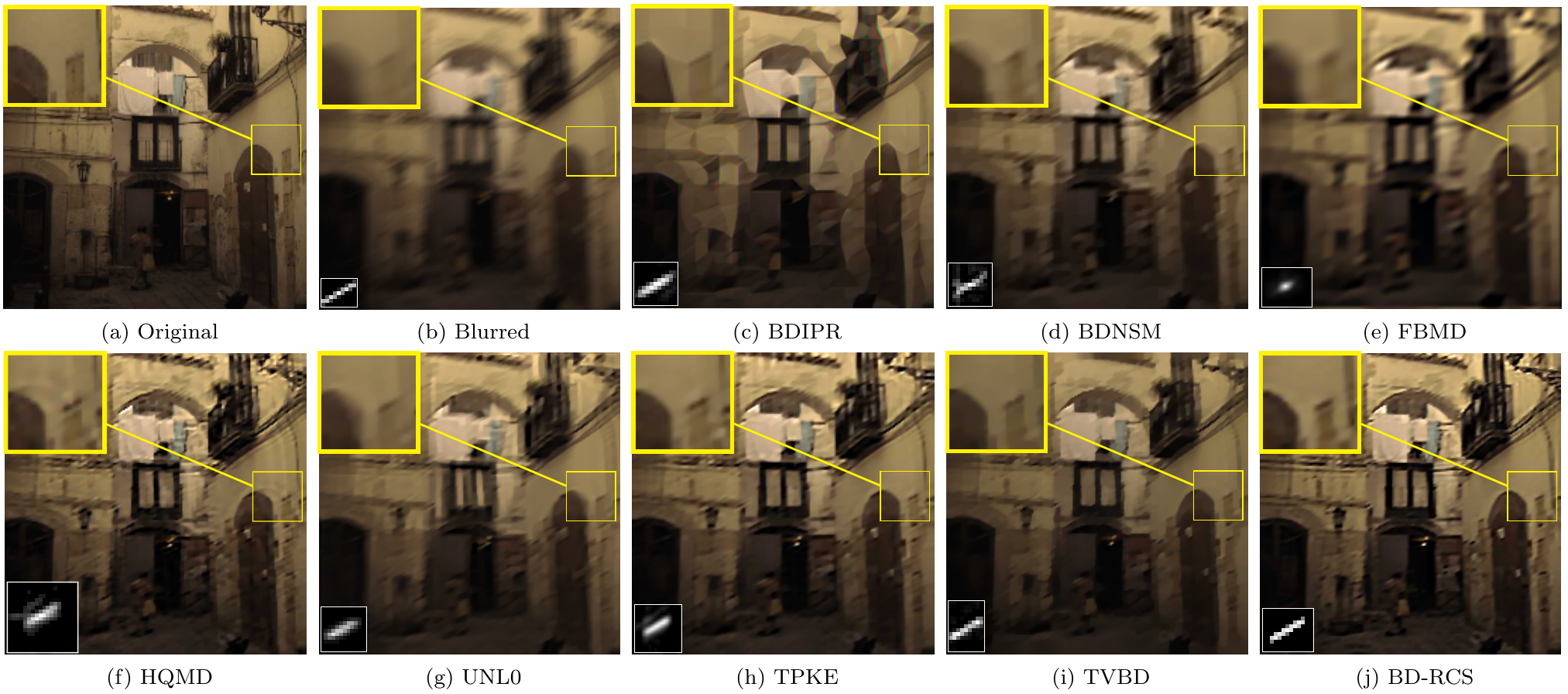}\vspace{-.4cm}
				\caption{a) Original image, b) Original PSF ($l=10, \theta=30$) and blured image, c) BDIRP \cite{Michaeli2014}; SNR = 15.40 dB, d) BDNSM \cite{Krishnan11}; SNR=19.17 dB, e) FBMD \cite{Cai12}; SNR=16.93 dB, f) HQMD \cite{Shan08}; SNR=18.01 dB, g) UNL0 \cite{Xu_2013_CVPR}; SNR=17.60, h) TPKE \cite{Xu2010}; SNR = 19.01 dB, i) TVBD \cite{Perrone16}; SNR = 20.31 dB, and j) BD-RCS; SNR = 22.01 dB.\vspace{-1cm}} \label{fig:row}
			\end{center}\vspace{-0cm}
		\end{figure*}
		{\color{white}{.}}

		\setcounter{figure}{7}
		\begin{figure}[h]
			\begin{center}
				\subfloat[]{\includegraphics[width=28.5mm]{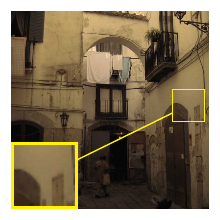} \label{fig:6a}}
				\subfloat[]{\includegraphics[width=28.5mm]{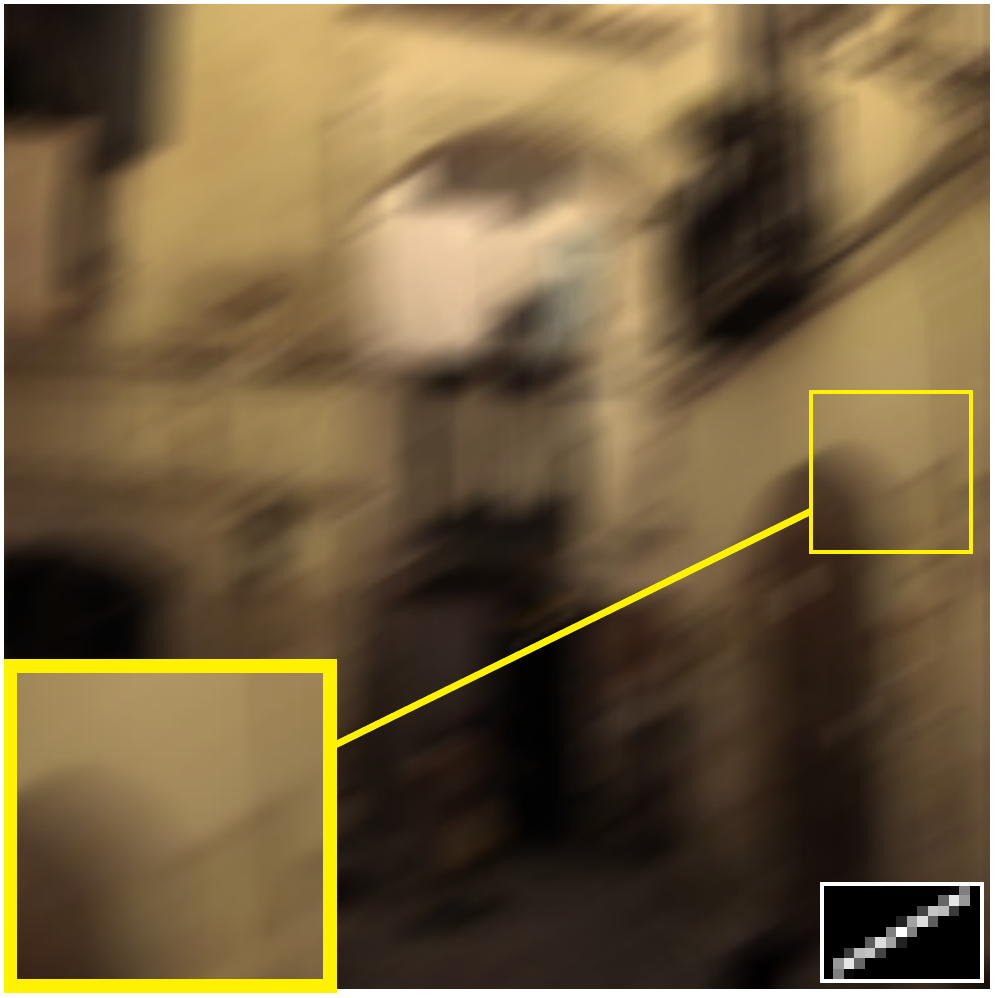} \label{fig:6b}}
				\subfloat[]{\includegraphics[width=28.5mm]{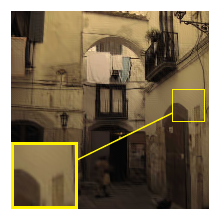} \label{fig:6c}}\\\vspace{-0.3cm}
				\subfloat[]{\includegraphics[width=28.5mm]{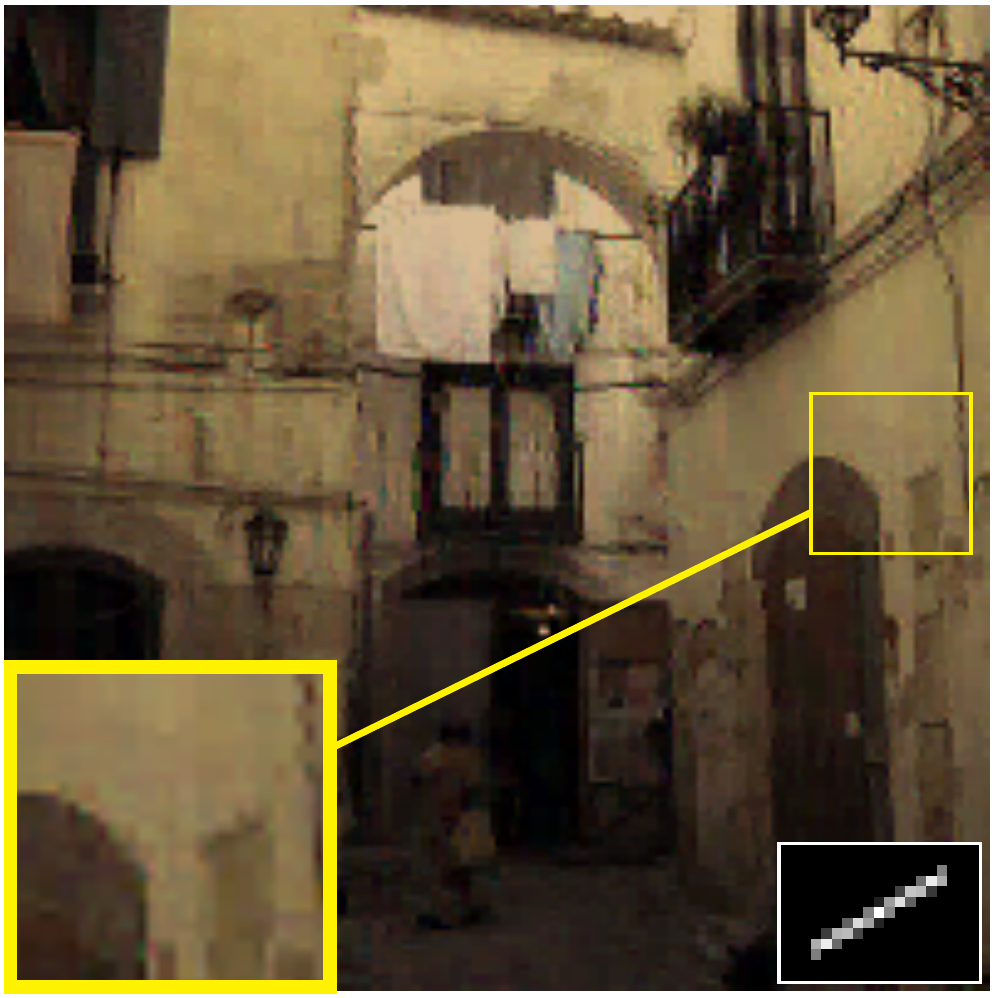} \label{fig:r3a}}
				\subfloat[]{\includegraphics[width=28.5mm]{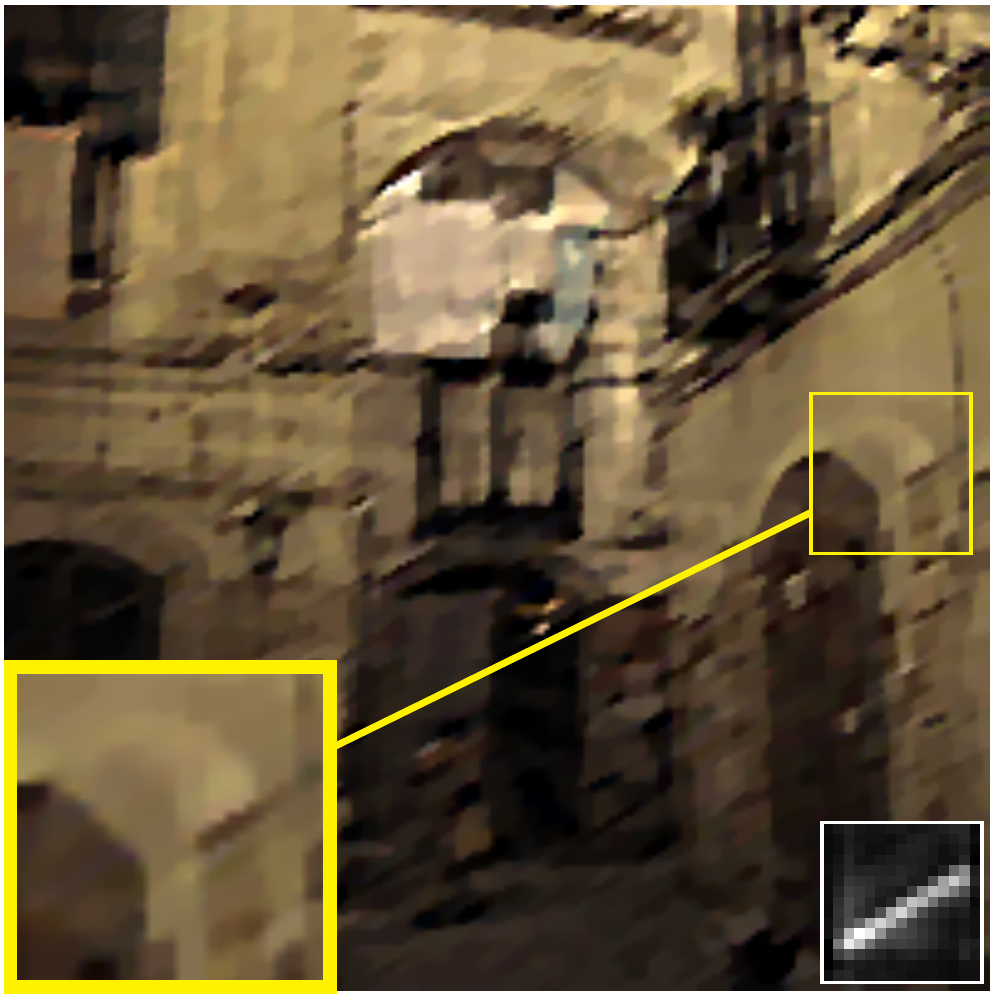} \label{fig:r3c}}
				\subfloat[]{\includegraphics[width=28.5mm]{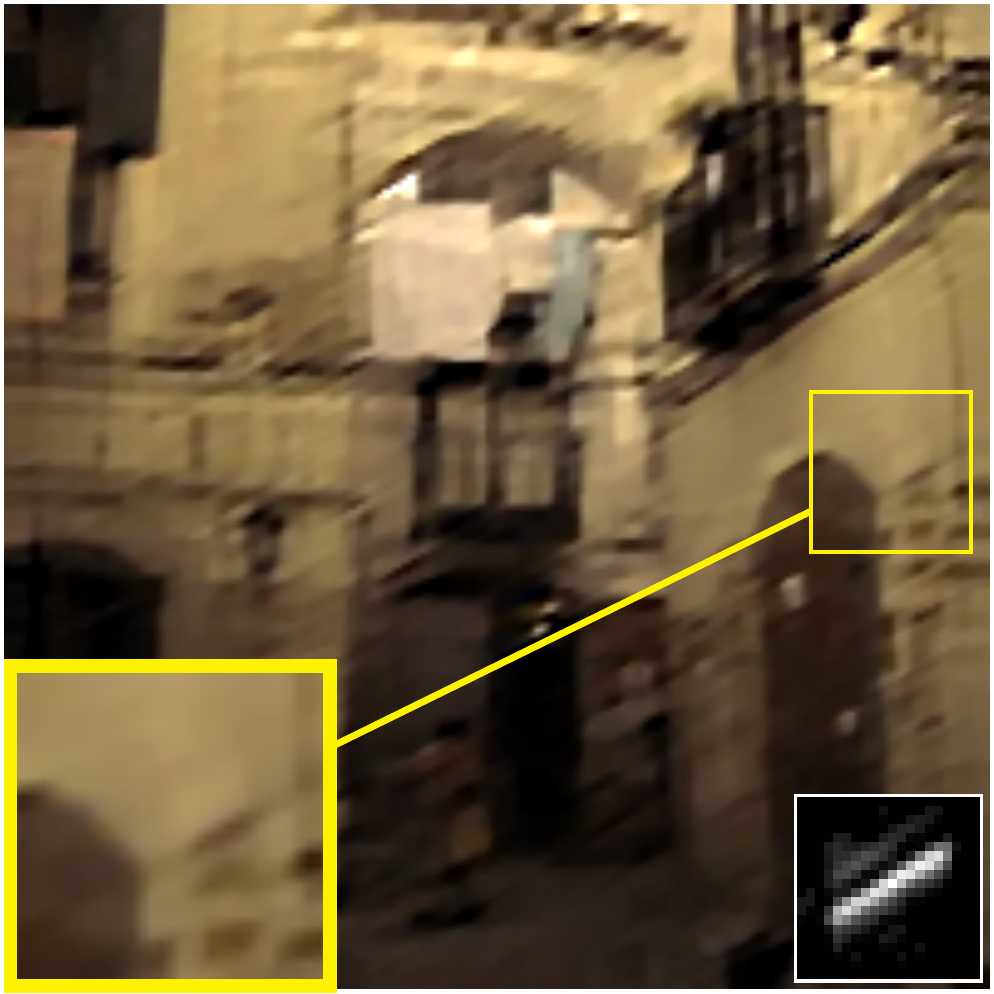} \label{fig:r3e}}\vspace{-0.1cm}
				\caption{a) Original image (Im3), b) Original PSF ($l=15, \theta=30$) and blured image, c) BM3D non-blind deblurring \cite{BM3Dblur}; SNR = 27.48 dB, ISNR = 13.54 dB, d) BD-RCS; SNR = 21.06 dB, ISNR = 7.14 dB, e) TVBD; SNR = 14.06 dB, ISNR = 0.14 dB, f) TPKE; SNR = 16.44 dB, ISNR = 2.52 dB.}\vspace{-8mm}
				\label{fig:r3}
			\end{center}\vspace{-0.5cm}
		\end{figure}
		
		\begin{figure}[h]
			\begin{center}
				\subfloat[]{\includegraphics[width=28.5mm]{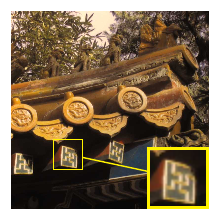} \label{fig:3a}}
				\subfloat[]{\includegraphics[width=28.5mm]{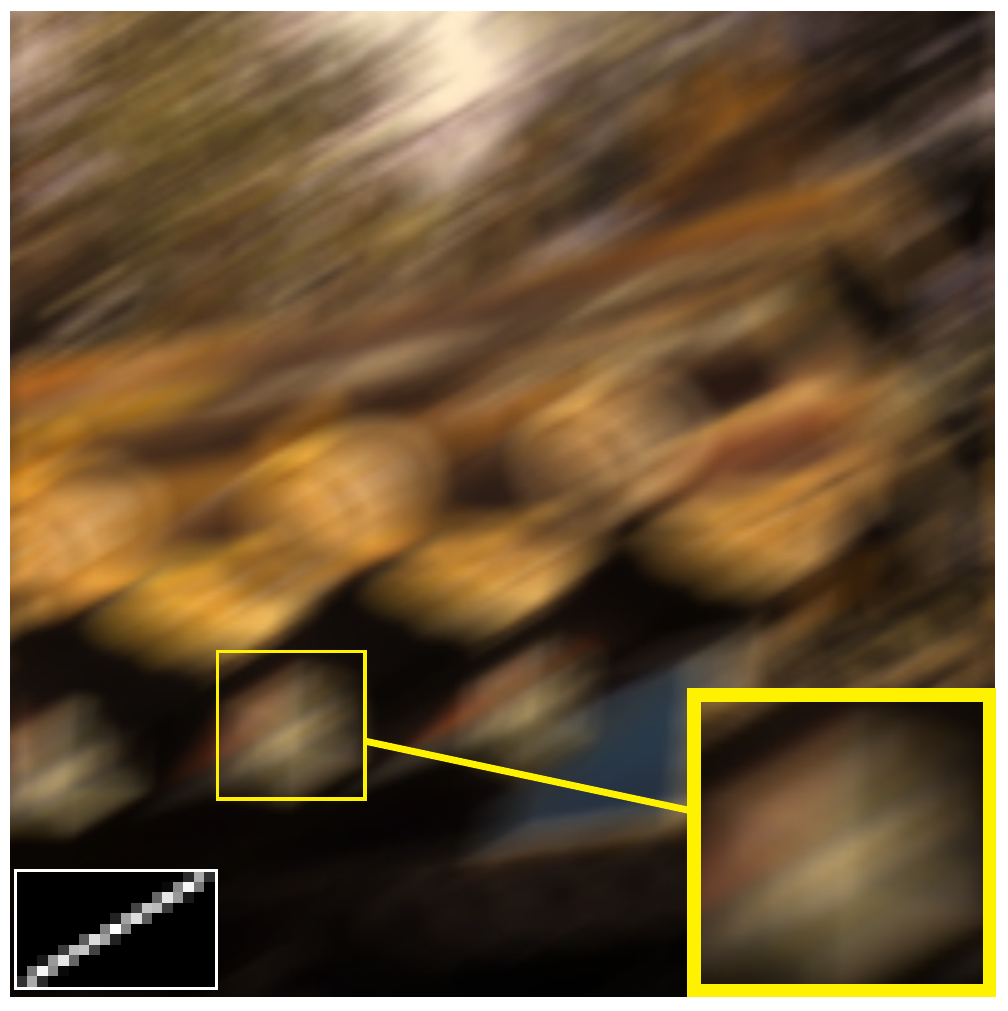} \label{fig:3b}}
				\subfloat[]{\includegraphics[width=28.5mm]{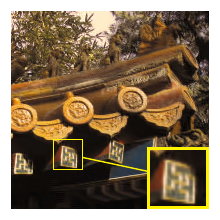} \label{fig:3c}}\\\vspace{-3mm}
				\subfloat[]{\includegraphics[width=28.5mm]{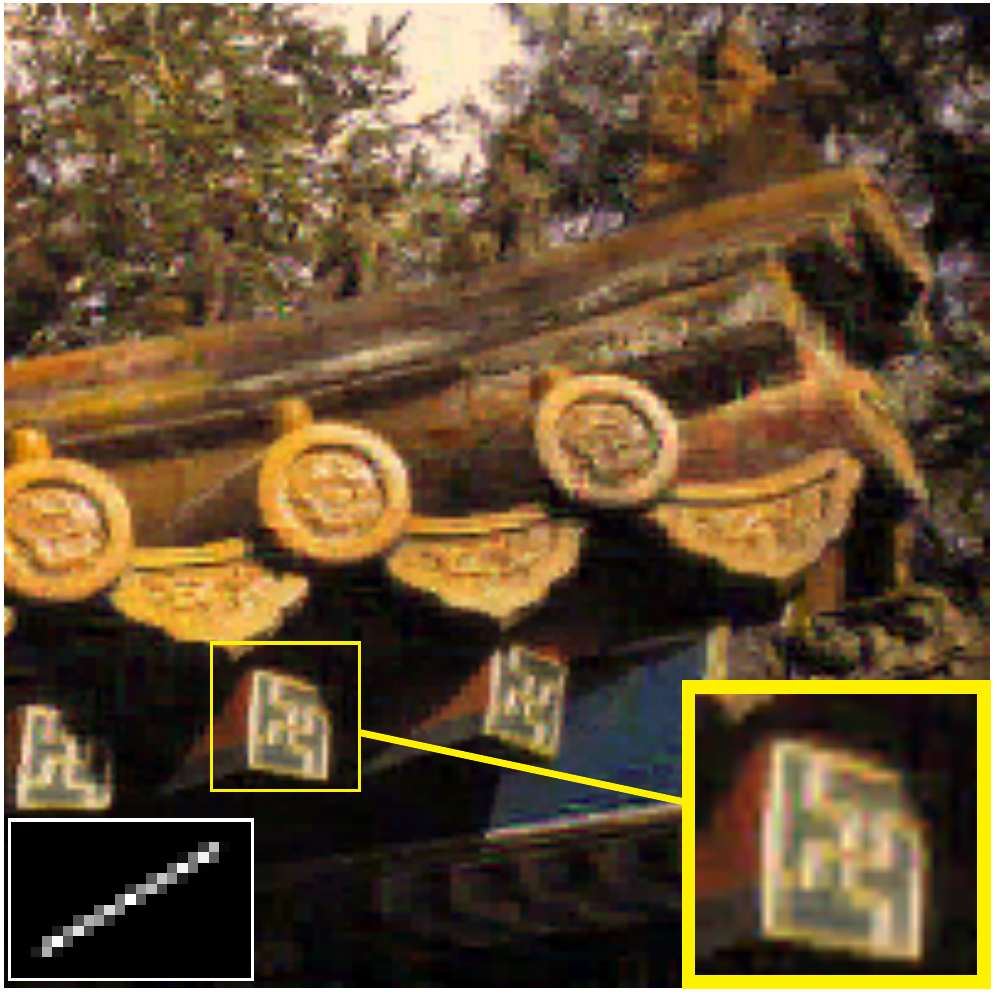} \label{fig:r1a}}
				\subfloat[]{\includegraphics[width=28.5mm]{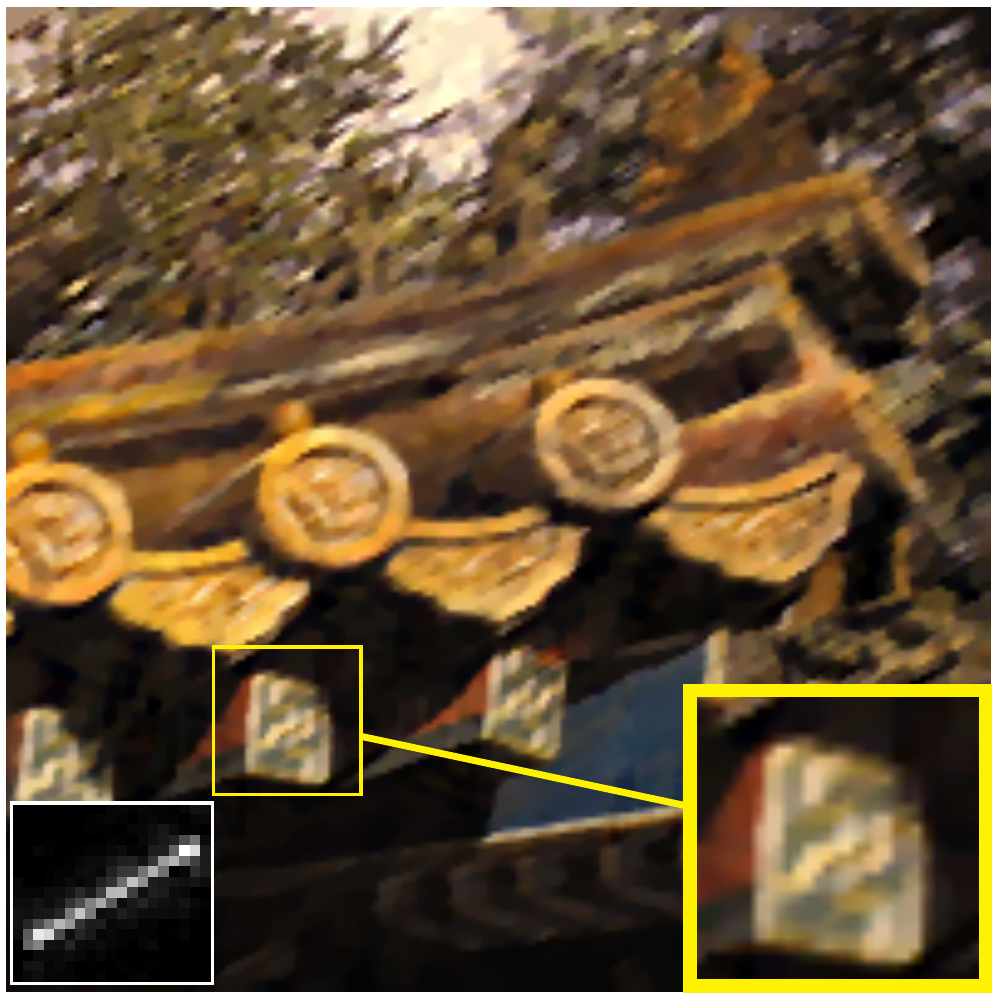} \label{fig:r1b}}
				\subfloat[]{\includegraphics[width=28.5mm]{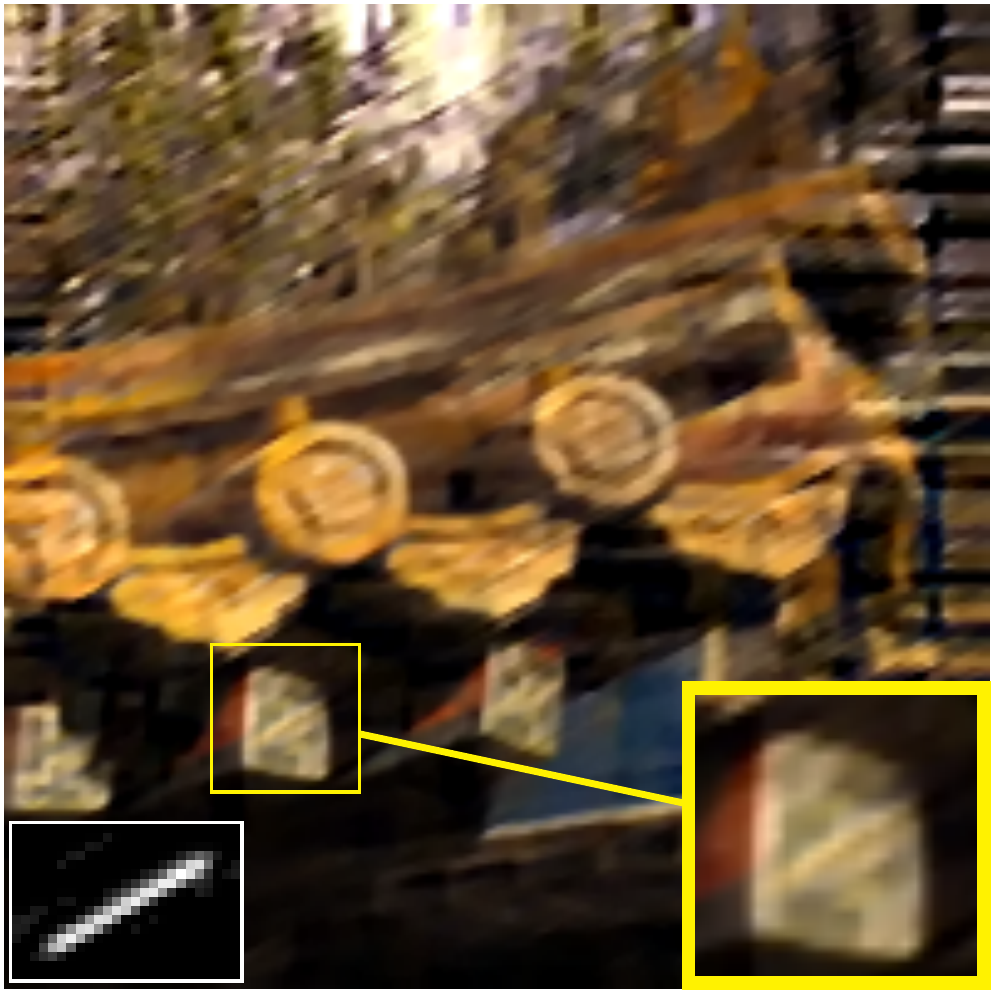} \label{fig:r1c}}\vspace{-1.5mm}
				\caption{a) Original image (Im4), b) Original PSF ($l=20, \theta=30\degree$) and blurred image, c) BM3D non-blind deblurring \cite{BM3Dblur}; SNR = 23.36 dB, ISNR = 12.84 dB, d) BD-RCS; SNR = 16.95 dB, ISNR = 6.40 dB, e) TVBD; SNR = 15.34 dB, ISNR = 4.79, f) TPKE; SNR = 12.19 dB, ISNR = 1.64 dB.\vspace{-7mm}}
				\label{fig:3}
			\end{center}\vspace{-0.5cm}
		\end{figure}
		
		%%%%%%%%%%%%%%%%%%%%%%%%%%%%%%%%%%%%%%%%%%%%%%%%%%%%%%%%
		\clearpage 
		\onecolumn
		\section{Convergence Analysis of BD-RCS}
		In the following analysis, we assume that $r$ is sufficiently large so that all the factorizations over $\bX = \bZ\bH^T$ are well defined; a conservative bound (depending only on $K,N,L$) for $r$ may be obtained using the methods discussed in~\cite{burer_local_2005}, etc. We also assume $\mathcal{A}$ is an operator satisfying the following property: $\|\mathcal{A}(\bX)\|_2^2 \geq \delta\|\bX\|_F^2$ for some $\delta>0,\forall~\bX\in\bR^{K\times N} \text{and rank}(\bX) \leq 2r$. This property is a weaker form of the restricted isometry property (RIP) discussed in~\cite{recht_guaranteed_2010}, and has been employed in many previous works such as~\cite{jain_low-rank_2013}. We further assume that L-BFGS succeeds in finding a local minima $(\bZ^{k+1},\bH^{k+1})$ to $\cL_{\sigma_k,\cA}(\bZ,\bH,\bw_k;\balpha_k),~\forall~k$.
		
		We will prove the convergence of Algorithm 1 that solves (5); proof corresponding to the algorithm solving (6) can be derived analogously. For ease of analysis (and to compensate for the scaling between (5) and (7) in the paper), we let $f^*$ be \emph{twice of} the optimal cost of (5) in paper and let $\bX^\ast=\bZs\bHs^T$ be its minimizer. We further define 
		
		\begin{align}
		f(\bZ,\bH,\bw):=\|\bZ\|_F^2+\|\bH\|_F^2+\sum_{i=1}^K\left(\bw_i+\lambda^2\frac{\|\be_i^T\bZ\bH^T\|_2^2}{\bw_i}\right). 
		\end{align}
		We also define $\bws$ via $\bw^\ast_i=\lambda\|\bZs\bHs^T\|_2$.
		
		\begin{proposition}
			Every local minima $(\bar{\bZ},\bar{\bH})$ of $\cL_{\sigma,\cA}(\bZ,\bH,\bw;\balpha)$ globally minimizes the Lagrangian $\cL_{0,\cA}(\bZ,\bH,\bw;\widehat{\balpha})$, where $\widehat{\balpha} = \balpha - \sigma(\cA(\bZ\bH^T) - \yhat)$.
			\label{prop1}
		\end{proposition}
		
		\begin{proof}
			Following the same reasoning as in Proposition 2.3 in~\cite{burer_local_2005}, $(\bar{\bZ},\bar{\bH})$ minimizes the following problem through the mapping $\bX = \bZ\bH^T$, $\bV = \bZ\bZ^T$, and $\bW = \bH\bH^T$:
			\begin{align}
				&\min_{\bX, \bW, \bV}\tr(\bV) + \tr(\bW) + \lambda^2\sum_{i=1}^{K}\frac{\|\be_i^T\bX\|_2^2}{\bw_{i}} + 2\langle \balpha, \yhat - \cA(\bX)\rangle + \sigma\|\cA(\bX)-\yhat\|^2 \nonumber\\
				&\text{subject to}\quad\begin{bmatrix}
					\bV & \bX\\
					\bX^T & \bW
				\end{bmatrix}\succeq\mathbf{0},
			\end{align}
			
			and in turn minimizes the following dual problem by the mapping $\bX = \bZ\bH^T$:
			\begin{align}
				&\min_{\bX} 2\|\bX\|_{*} + \lambda^2\sum_{i=1}^{K}\frac{\|\be_i^T\bX\|_2^2}{\bw_{i}} + 2\langle \balpha, \yhat - \cA(\bX)\rangle + \sigma\|\cA(\bX)-\yhat\|^2.
			\end{align}
			Therefore, $0\in\partial\|\bar{\bZ}\bar{\bH}^T\|_{*} + \lambda^{2}\mathrm{diag}(\bw)^{-1}\bar{\bZ}\bar{\bH}^T - \cA^{*}(\widehat{\balpha})$ where $\partial$ denotes the subdifferential and $\cA^\ast$ is the adjoint operator to $\cA$. This implies:
			\begin{align}
				&\|\bar{\bZ}\bar{\bH}^T\|_{*} + \frac{\lambda^2}{2}\sum_{i=1}^{K}\frac{\|\be_i^T\bar{\bZ}\bar{\bH}^T\|_2^2}{\bw_{i}} - \langle \cA^{*}(\widehat{\balpha}), \bar{\bZ}\bar{\bH}^T\rangle\nonumber\\
				&\leq\|\bX\|_{*} + \frac{\lambda^{2}}{2}\sum_{i=1}^{K}\frac{\|\be_i^T\bX\|_2^2}{\bw_{i}} - \langle\cA^{*}(\widehat{\balpha}), \bX\rangle \quad \forall~\bX\in\bR^{K\times N}.
			\end{align}
			In particular, letting $\bX = \bZ\bH^T$ gives
			\begin{align}
				&\|\bar{\bZ}\bar{\bH}^T\|_{*} + \frac{\lambda^2}{2}\sum_{i=1}^{K}\frac{\|\be_i^T\bar{\bZ}\bar{\bH}^T\|_2^2}{\bw_{i}} - \langle \cA^{*}(\widehat{\balpha}), \bar{\bZ}\bar{\bH}^T\rangle\nonumber\\
				&\leq\|\bZ\bH^T\|_{*} + \frac{\lambda^{2}}{2}\sum_{i=1}^{K}\frac{\|\be_i^T\bZ\bH^T\|_2^2}{\bw_{i}} - \langle\cA^{*}(\widehat{\balpha}), \bZ\bH^T\rangle \forall~\bH\in\bR^{N\times r}, \bZ\in\bR^{K\times r}. \label{star}
			\end{align}
			On the other hand, since $\nabla_\bZ\cL_{\sigma,\cA}(\bar{\bZ},\bar{\bH},\bw;\bar{\balpha})=0$, $\nabla_\bH\cL_{\sigma,\cA}(\bar{\bZ},\bar{\bH},\bw;\bar{\balpha})=0$:
			\begin{align}
				\bar{\bZ}-\cA^\ast(\widehat{\balpha})\bar{\bH}+\lambda^2\diag(\bw)^{-1}\bar{\bZ}\bar{\bH}^T\bar{\bH}&=0,\label{1}\\
				\bar{\bH}^T-\bar{\bZ}^T\cA^\ast(\widehat{\balpha})+\lambda^2\bar{\bZ}^T\diag(\bw)^{-1}\bar{\bZ}\bar{\bH}^T&=0.\label{2}
			\end{align}
			Left multiplying (\ref{1}) by $\bar{\bZ}^T$, right multiplying (\ref{2}) by $\bar{\bH}$ and subtracting the two gives $\bar{\bZ}^T\bar{\bZ} = \bar{\bH}^T\bar{\bH}$; thus $\bar{\bZ}$, $\bar{\bH}$ admits the following singular value decomposition:
			\begin{align}
				&\bar{\bZ} = \bU_1\bS\bV^T, \bar{\bH} = \bU_2\bS\bV^T, \bU_1\in\bR^{K\times r}, \nonumber\\
				&\bU_2\in\bR^{N\times r}, \bS\in\bR^{r\times r}, \bV\in\bR^{r\times r},
			\end{align}
			and therefore
			\begin{align}
				&\|\bar{\bZ}\bar{\bH}^T\|_{*} = \|\bU_1\bS^2\bU^T\|_{*} = \sum_{i=1}^{r}\bS^2_i = \frac{1}{2}\left(\|\bZ\|_F^2 + \|\bH\|_F^2\right) \quad \mathrm{where}~\bS = \begin{bmatrix}
					\bS_1 & & & \\
					& \bS_2 & & \bigzero\\
					\bigzero&  &\ddots & &\\
					& & & \bS_r
				\end{bmatrix}.\label{3}
			\end{align}
			Following the same procedures on $\nabla_\bZ\cL_{0,\cA}(\bZs,\bHs,\bws;\balphas)=0$ and $\nabla_\bH\cL_{0,\cA}(\bZs,\bHs,\bws;\balphas)=0$, we can obtain
			\begin{equation}
				\|\bZs\bHs^T\|_\ast=\frac{1}{2}\left(\|\bZs\|_F^2+\|\bHs\|_F^2\right).
				\label{eqn:start_eq}
			\end{equation}
			On the other hand, using the Arithmetic-Geometric Mean Inequality~\cite{bhatia_matrix_1997}:
			\begin{align}
				\sum_{i}s_i(\bZ\bH^T)&\leq\frac{1}{2}\sum_{i}\left[s_i(\bZ^T\bZ)^2 + s_i(\bH^T\bH)^2\right]\nonumber\\
				&= \frac{1}{2}\left[\sum_{i}s_i(\bZ^T\bZ)^2\right] + \left[\sum_{i}s_i(\bH^T\bH)^2\right]\nonumber\\
				&= \frac{1}{2}\left(\|\bZ\|_F^2 + \|\bH\|_F^2\right)\label{4},
			\end{align}
			where $s_i(\bX)$ is the $i$-th singular value of $\bX$. Plugging (\ref{3}) and (\ref{4}) into (\ref{star}) gives $\cL_{0,\cA}(\bar{\bZ},\bar{\bH},\bw;\widehat{\balpha}) \leq \cL_{0,\cA}(\bZ,\bH,\bw;\widehat{\balpha})$.
		\end{proof}
		
		\begin{proposition}
			$\|\bw^k - \bw^*\|_2^2 \leq 2\left(\lambda^2\|\bZ^k\bH^{kT} - \bZ^*\bH^{*T}\|_F^2+\frac{K\varepsilon_0^2}{(k+2)^4}\right)$.
			\label{prop2}
		\end{proposition}
		\begin{proof}
			Using the mean-value theorem
			\begin{align}
				|\bw_i^k - \bw_i^*| &=\left|\lambda\|\be_i^T\bZ^k\bH^{kT}\|_2 +\frac{\varepsilon_0}{(k+2)^2}- \lambda\|\be_i^T\bZ^*\bH^{*T}\|_2\right|\nonumber\\
				&= \left|\lambda\frac{\be_i^T\bZ_\theta\bH_\theta^{T}}{\|\be_i^T\bZ_\theta\bH_\theta^{T}\|_2}\left(\bZ^k\bH^{kT} - \bZ^*\bH^{*T}\right)\be_i+\frac{\varepsilon_0}{(k+2)^2}\right|\nonumber\\
				&\leq \lambda \left\|\left(\bZ^k\bH^{kT} - \bZ^*\bH^{*T}\right)\be_i\right\|_2+\frac{\varepsilon_0}{(k+2)^2} ~~ (\text{Cauchy-Schwarz}),\label{eqn:wi}
			\end{align}
			where $\bZ_\theta = \theta\bZ^* + (1-\theta)\bZ^k$, $\bH_\theta = \theta\bH^* + (1-\theta)\bH^k$, $\theta\in(0,1)$. Therefore 
			\begin{align*}
				\|\bw^k - \bw^*\|_2^2 &= \sum_{i=1}^{K}|\bw_i^k - \bw_i^*|^2\\
				&\leq \sum_{i=1}^{K}2\left(\lambda^2 \|\left(\bZ^k\bH^{kT} - \bZ^*\bH^{*T}\right)\be_i\|^2_2+\frac{\varepsilon_0^2}{(k+2)^4}\right)\\
				&=2\left(\lambda^2\|\bZ^k{\bH^k}^T-\bZs\bHs^T\|_F^2+\frac{\varepsilon_0^2}{(k+2)^4}\right).
			\end{align*}
		\end{proof}
		
		\begin{proposition}
			$f^* \leq \cL_{0,\cA}(\bZ,\bH,\bw;\balphas), \forall~\bZ\in\bR^{K\times r}, \bH\in\bR^{N\times r},\mathrm{and}~\bw\in\bR^{K}$.
			\label{prop3}
		\end{proposition}
		\begin{proof}
			Since $\bX^* = \bZ^*\bH^{*T}$ minimizes the convex program (5) in paper, it also minimizes its Lagrangian, and we thus have
			\begin{align}
				f^* &= \|\bZ^*\bH^{*T}\|_* + \lambda\|\bZ^*\bH^{*T}\|_{2,1} \leq \|\bX\|_* + \lambda\|\bX\|_{2,1} + \langle \balphas, \yhat - \cA(\bX)\rangle,\forall\bX\in\bR^{K\times N}.
			\end{align}
			By taking $\bX = \bZ\bH^{T}$, the RHS becomes 
			\begin{align}
				&\|\bZ\bH^T\|_{*} + \lambda\|\bZ^*\bH^{*T}\|_{2,1} + \langle \balphas, \yhat - \cA(\bZ\bH^T)\rangle\nonumber\\
				&\leq \frac{1}{2}\left(\|\bZ\|_F^2 + \|\bH\|_F^2\right) + \frac{1}{2}\sum_{i=1}^{K}\left(\bw_i+ \lambda^2\frac{\|\be_i^T\bZ\bH^T\|_2^2}{\bw_{i}}\right) + \langle \balphas, \yhat - \cA(\bZ\bH^T)\rangle \nonumber\\
				&= \frac{1}{2}\cL_{0,\cA}(\bZ,\bH,\bw;\balphas),
			\end{align} 
			where we used the inequality $\|\bZ\bH^T\|_*\leq\frac{1}{2}\left(\|\bZ\|_F^2 + \|\bH\|_F^2\right)$ (as in the proof of the Proposition 1) and $\bw_i + \lambda^2\frac{\|\be_i^T\bZ\bH^T\|_2^2}{\bw_{i}}\geq2\sqrt{\lambda^2\bw_i\frac{\|\be_i^T\bZ\bH^T\|_2^2}{\bw_{i}}} = 2\lambda\|\be_i^T\bZ\bH^T\|_2$.
		\end{proof}
		
		\begin{theorem}
		The sequence $(\bZ^k, \bH^k)$ generated by Algorithm 1 converges to the optimal solution of (5) in the paper in the sense of $\lim\limits_{k}\bZ^k\bH^{kT} = \bZ^*\bH^{*T} = \bX^*$.
		\end{theorem}
		\begin{proof}
			Let $\bgamma^{k} = \yhat - \cA(\bZ^k\bH^{kT})$. By Proposition~\ref{prop1}, $\cL_{0,\cA}(\bZ^{k+1},\bH^{k+1},\bw^{k};\balpha^{k+1}) \leq \cL_{0,\cA}(\bZ^*,\bH^*,\bw^k;\balpha^{k+1})$, i.e.
			\begin{align}
				f(\bZ^{k+1},\bH^{k+1},\bw^{k}) + \langle\balpha^{k+1},\bgamma^{k+1}\rangle\leq f(\bZ^*,\bH^*,\bw^k).
				\label{5}
			\end{align}
			
			By Proposition~\ref{prop3}, $f^* \leq \cL_{0,\cA}(\bZ^{k+1},\bH^{k+1},\bw^{k};\balphas)$, i.e.
			\begin{align}
				f^* \leq f(\bZ^{k+1},\bH^{k+1},\bw^{k}) + \langle\balpha^{k+1},\bgamma^{k+1}\rangle.
				\label{6}
			\end{align}
			
			Adding (\ref{5}) and (\ref{6}) gives
			\begin{align}
				\langle\balpha^{k+1}-\balphas,\bgamma^{k+1}\rangle\leq f(\bZ^*,\bH^*,\bw^k) - f^*. \label{7}
			\end{align}
			
			Recall $\balpha^{k}=\balpha^{k+1}-\sigma_k\bgamma^{k+1}$; thus 
			\begin{align}
				&\|\balpha^k - \balphas\|_2^2 - \|\balpha^{k+1} - \balphas\|^2_2 \nonumber\\
				&= \sigma_k^2\|\bgamma^{k+1}\|_2^2 - 2\sigma_k\langle\balpha^{k+1}-\balphas,\bgamma^{k+1}\rangle\nonumber\\
				&\geq\sigma_k^2\|\bgamma^{k+1}\|_2^2 - 2\sigma_k \left(f(\bZ^*,\bH^*,\bw^k) - f^*\right) \quad (\text{using (\ref{7})}).
				\label{8}
			\end{align}
			
			On the other hand, since $\bZ^k\bH^{kT} - \bZ^*\bH^{*T}$ has rank at most $2r$, we may invoke the RIP property of $\cA$ to get
			\begin{align}
				\|\bgamma^k\|_2^2=\|\cA(\bZ^k\bH^{kT} - \bZ^*\bH^{*T})\|^2_2 \geq \delta\|\bZ^k\bH^{kT} - \bZ^*\bH^{*T}\|_F^2,
			\end{align}
			
			and by letting $\varepsilon_k=\frac{\varepsilon_0}{(k+1)^2}$, we have
			\begin{align}
				&f(\bZ^*,\bH^*,\bw^k) - f^*\nonumber\\
				&= \sum_{i=1}^{K}\left(\bw_i^k + \frac{\bw_i^{*2}}{\bw_i^k} - 2\bw_i^*\right)\qquad(\text{using~(\ref{eqn:start_eq})}) \nonumber\\
				&= \sum_{i=1}^{K}\frac{1}{2}\frac{(\bw_i^k - \bw_i^*)^2}{\bw_i^k} \nonumber\\
				&\leq\frac{\|\bw^k - \bw^*\|^2_2}{2\varepsilon_k}\qquad(\text{since $w_i^k\geq\varepsilon_k$}) \nonumber\\
				&\leq\frac{\lambda^2}{\varepsilon_k}\|\bZ^k\bH^{kT} - \bZ^*\bH^{*T}\|_F^2+K\varepsilon_k\qquad(\text{using Proposition 2})\nonumber\\
				&\leq\frac{\lambda^2(k+1)^2}{\varepsilon_0\delta}\|\bgamma^k\|_2^2+\frac{K\varepsilon_0}{(k+1)^2}.
				\label{9}
			\end{align}
			
			Plugging (\ref{9}) into (\ref{8}), we get
			\begin{align}
				&\|\balpha^k - \balphas\|^2_2 - \|\balpha^{k+1} - \balphas\|_2^2 \geq \sigma_k^2\|\bgamma^{k+1}\|_2^2 \nonumber\\
				&- \frac{2\sigma_k\lambda^{2}(k+1)^2}{\varepsilon_0\delta}\|\bgamma^k\|_2^2-\frac{2\sigma_kK\varepsilon_0}{(k+1)^2}.
			\end{align}
			
			And by left division over $\sigma_k$, we have
			\begin{align}
				&\frac{\|\balpha^k-\balpha^\ast\|_2^2}{\sigma_k}-\frac{\|\balpha^{k+1}-\balpha^\ast\|_2^2}{\sigma_{k+1}}\nonumber\\
				&\geq\frac{\|\balpha^k - \balphas\|^2_2}{\sigma_k} - \frac{\|\balpha^{k+1} - \balphas\|_2^2}{\sigma_k} \nonumber\\
				&\geq \sigma_k\|\bgamma^{k+1}\|_2^2 - \frac{2\lambda^{2}(k+1)^2}{\varepsilon_0\delta}\|\bgamma^k\|_2^2-\frac{2K\varepsilon_0}{(k+1)^2}.
			\end{align}
			
			Since $\sigma_k = \sigma_0\rho^k$, $\sigma_k>\frac{2\lambda^2\rho(k+2)^2}{\varepsilon_0\delta}+1$ for $k\geq\bar{k}$. Summing over $\bar{k}$ to $\infty$ gives 
			\begin{align}
				\frac{\|\balpha^{\bar{k}} - \balphas\|_2^2}{\sigma_{\bar{k}}}&\geq\sum_{k=\bar{k}}^{\infty}\left[\sigma_k - \frac{2\lambda^{2}(k+2)^2}{\varepsilon_0\delta}\right]\|\bgamma^{k+1}\|_2^2 -\nonumber\\
				&~~~~\sigma_{\bar{k}}\frac{2\lambda^{2}(\bar{k}+1)^2}{\varepsilon_0\delta}\|\bgamma^{\bar{k}}\|_2^2-\sum_{k=\bar{k}}^{\infty}\frac{2K\varepsilon_0}{(k+1)^2}\nonumber\\
				&\geq\sum_{k=\bar{k}}^{\infty}\|\bgamma^{k+1}\|_2^2 - \sigma_{\bar{k}}\frac{2\lambda^{2}(\bar{k}+1)^2}{\varepsilon_0\delta}\|\bgamma^{\bar{k}}\|_2^2-\nonumber\\
				&~~~\sum_{k=\bar{k}}^\infty\frac{2K\varepsilon_0}{(k+1)^2}.
			\end{align}
			
			Since the left hand side is bounded and $\sum_{k=\bar{k}}^\infty\frac{1}{(k+1)^2}$ converges, $\sum_{k=\bar{k}}^{\infty}\sigma_k\|\bgamma^{k+1}\|^2$ converges and thus $\bgamma^{k+1}\rightarrow 0$. Moreover, as $\|\bZ^k\bH^{kT} - \bZ^*\bH^{*T}\|_F^2\leq\frac{1}{\delta}\|\cA(\bZ^k\bH^{kT} - \bZ^*\bH^{*T})\|^2_2$, we have $\lim\limits_{k}\bZ^k\bH^{kT} = \bZ^*\bH^{*T}$.
		\end{proof}
		
		Finally, we note that in general it is impossible to prove $\lim_k\bZ_k=\bZ^\ast$ and $\lim_k\bH_k=\bH^\ast$ separately due to the fact that $(\bZs,\bHs)$ cannot be uniquely identified. To understand this point, let $\bQ\in\bR^{r\times r}$ be any orthogonal matrix, then $(\bZs\bQ,\bHs\bQ)$ also certifies as a minimizer to (7) in the paper. However, the dominating singular vector remains the same except for possible sign changes, and under the context of blind deblurring such ambiguity is inconsequential.

		% -------------------------------------------------------------------------
		%\clearpage
		%%\pagebreak
		%%		\begin{spacing}{0.8}%1
		%%			\setlength{\bibsep}{0pt}%1pt
		%\bibliographystyle{IEEEbib}
		%\bibliography{PhdReferences}
		%%		\end{spacing}
		
	\end{document}